\documentclass[twoside,leqno,twocolumn]{article}

\usepackage[letterpaper]{geometry}

\usepackage{ltexpprt}
\usepackage{hyperref}

\usepackage{bm}
\usepackage[edges]{forest}
\usepackage{stmaryrd}
\usepackage{makecell}
\usepackage{subcaption}
\usepackage{calc}
\usepackage[backend=biber]{biblatex}
\usepackage[all]{nowidow}

\DeclareFieldFormat*{title}{\mkbibemph{\MakeSentenceCase{\MakeLowercase{#1}}}}
\DeclareFieldFormat*{citetitle}{\mkbibemph{#1}}
\DeclareFieldFormat{journaltitle}{#1}
\DeclareFieldFormat{booktitle}{#1}

\renewbibmacro*{in:}{%
  \ifentrytype{article}
    {\printtext{\bibstring{in} }}
    {%
      \ifentrytype{inproceedings}
        {\printtext{\bibstring{in} }}
        {}
    }%
}
\newbibmacro*{pubinstorg+location+date}[1]{%
  \printlist{#1}%
  \newunit
  \printlist{location}%
  \newunit
  \usebibmacro{date}%
  \newunit}

\renewbibmacro*{publisher+location+date}{\usebibmacro{pubinstorg+location+date}{publisher}}
\renewbibmacro*{institution+location+date}{\usebibmacro{pubinstorg+location+date}{institution}}
\renewbibmacro*{organization+location+date}{\usebibmacro{pubinstorg+location+date}{organization}}

\usepackage{amsmath,amssymb,amsfonts}
\usepackage{tabularx}
\usepackage{float}
\usepackage{mathtools}

\newcommand{\reminder}[1]{{\textsf{\small \textcolor{orange}  {[#1]}}}}

\newcommand{\sos}[1]{{\text{\textcolor{red}  {[#1]}}}}

\newcommand{\hide}[1]{}

\newcommand{\cleanup}{
    \renewcommand{\reminder}{\hide}
    \renewcommand{\sos}{\hide}
}

\DeclareMathOperator*{\argsup}{arg\,sup}
\DeclareMathOperator*{\arginf}{arg\,inf}

\DeclareMathOperator*{\diagonalize}{diag}
\DeclareMathOperator*{\trace}{Tr}
\DeclareMathOperator*{\rnk}{rank}
\DeclareMathOperator*{\nnzop}{nnz}

\newcommand{\mat}[1]{\mathbf{#1}}
\newcommand{\tensor}[1]{\bm {\mathcal{#1}}}

\newcommand{\diag}[1]{\diagonalize\left(#1\right)}  

\newcommand{\norm}[1]{\left|\left|#1\right|\right|}
\newcommand{\cp}[6]{\llbracket\mat{#1} {#4},\mat{#2} {#5},\mat{#3} {#6} \rrbracket}
\newcommand{\tr}[1]{\trace\left(#1\right)}
\newcommand{\myset}[1]{ {\mathcal{#1}}}
\newcommand{\rank}[1]{ \rnk\left({#1}\right)}
\newcommand{\nnz}[1]{\nnzop\left(#1\right)}

\newcommand{\method}[1]{GenClus#1\xspace}

\usepackage{xspace}

\usepackage{appendix}
\usepackage{amsmath}

\usepackage{etoolbox}
\makeatletter
\patchcmd{\hyper@makecurrent}{%
    \ifx\Hy@param\Hy@chapterstring
        \let\Hy@param\Hy@chapapp
    \fi
}{%
    \iftoggle{inappendix}{
        \@checkappendixparam{chapter}%
        \@checkappendixparam{section}%
        \@checkappendixparam{subsection}%
        \@checkappendixparam{subsubsection}%
        \@checkappendixparam{paragraph}%
        \@checkappendixparam{subparagraph}%
    }{}%
}{}{\errmessage{failed to patch}}

\newcommand*{\@checkappendixparam}[1]{%
    \def\@checkappendixparamtmp{#1}%
    \ifx\Hy@param\@checkappendixparamtmp
        \let\Hy@param\Hy@appendixstring
    \fi
}
\makeatletter

\newtoggle{inappendix}
\togglefalse{inappendix}

\apptocmd{\appendix}{\toggletrue{inappendix}}{}{\errmessage{failed to patch}}
\apptocmd{\subappendices}{\toggletrue{inappendix}}{}{\errmessage{failed to patch}}

\providecommand{\keywords}[1]{\noindent{\small\textbf{\textit{Keywords---}} #1}}

\usepackage{etoolbox}

\newbool{movecontent}
\setbool{movecontent}{true} 

\newcommand{\appendixcontent}{} 

\newcommand{\moveToAppendix}[2]{%
    \ifbool{movecontent}{%
        \gappto\appendixcontent{\subsection{#1}#2}%
    }{%
        #2%
    }%
}
\usepackage[switch]{lineno}

\cleanup
\begin{document}

\newcommand\relatedversion{}


\title{\Large Multi-View Spectral Clustering for Graphs with Multiple View Structures
}
\author{
    Yorgos Tsitsikas\thanks{University of California, Riverside. \{gtsit001@ucr.edu, epapalex@cs.ucr.edu\}}
    \and
    Evangelos E. Papalexakis\footnotemark[1]
    }

\date{}

\maketitle







\begin{abstract}\small\baselineskip=9pt 
    Despite the fundamental importance of clustering, to this day, much of
    the relevant research is still based on ambiguous foundations, leading
    to an unclear understanding of whether or how the various clustering
    methods are connected with each other. In this work, we provide an
    additional stepping stone towards resolving such ambiguities by
    presenting a general clustering framework that subsumes a series of
    seemingly disparate clustering methods, including various methods
    belonging to the widely popular spectral clustering framework. In fact,
    the generality of the proposed framework is additionally capable of
    shedding light to the largely unexplored area of multi-view graphs where
    each view may have differently clustered nodes. In turn, we propose
    \method: a method that is simultaneously an instance of this framework
    and a generalization of spectral clustering, while also being closely
    related to k-means as well. This results in a principled alternative to
    the few existing methods studying this special type of multi-view
    graphs. Then, we conduct in-depth experiments, which demonstrate that
    \method{} is more computationally efficient than existing methods, while
    also attaining similar or better clustering performance. Lastly, a
    qualitative real-world case-study further demonstrates the ability of
    \method{} to produce meaningful clusterings.
\end{abstract}
\vspace{\baselineskip}

\begin{refsection}
\section{Introduction.}
Theoretical and computational developments in linear algebra and graph
theory have set the foundations for the spectral clustering family
of methods
\cite{von2007tutorial}, \cite{ng2001spectral}, \cite{zhou2005learning}, \cite{klus2022koopman},
which has proven to be one of the most successful clustering paradigms.
However, although researchers have been defining novel and intricate graph
types that aim to capture more complex information compared to traditional
graphs \cite{magnani2021community}, extending spectral clustering to such
graphs is still not very well understood. Multi-view graphs are a popular
instance of such graphs, which consist of sets of simple graphs
called views. Each view consists of the same set of nodes, but a potentially different set
of edges. Prominent examples are time-evolving graphs
\cite{zhong2014evolution}, where different views emerge as edges appear
or disappear at different points in time, and multi-relational knowledge
graphs \cite{nickel2015review}, which describe the relationships between
entities, and where each relation type gives birth to a different view. The
main rationale for developing models for multi-view graphs is that they are 
often able to obtain embeddings that better capture the intricacies of the
graph structure as compared to considering individual views alone
\cite{papalexakis2013more}.  

\begin{figure}[h]
    \centering
    \includegraphics[width=\linewidth]{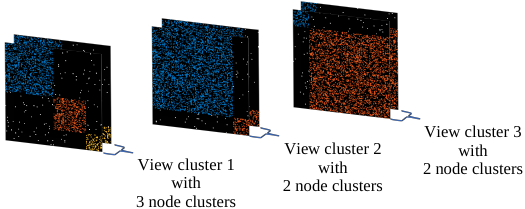}
    \caption{Example of multi-view graph with 6 views grouped into 3 view
    clusters, each corresponding to a different node clustering with
    3, 2 and 2 node clusters, respectively. Each view is visualized via its
    adjacency matrix and darker colors represent lower edge weights.}
    \label{fig:multi_structure_multi_relational_example}
\end{figure}
Note that, existing research on multi-view graphs has primarily focused on
graphs whose views are all part of the same underlying node clustering
structure \cite{zhou2007spectral}, \cite{liu2012multiview}.  However,
recently there have also been attempts to extend such methods to multi-view
graphs with multiple view structures \cite{ni2017comclus},
\cite{gujral2020beyond}, \cite{chen2019tensor}.  Multi-view graphs of this
type may contain multiple groups of views, each corresponding to a unique
clustering of nodes, as illustrated in
\autoref{fig:multi_structure_multi_relational_example}.  Also, although
there have been attempts to adapt spectral clustering to multi-view graphs
\cite{zhou2007spectral}, \cite{liu2012multiview}, to the best of our
knowledge no such extension exists for multi-view graphs with multiple view
structures. Additionally, since many relevant methods have been developed
independently of each other, it is not yet understood whether or how
they may be related to each other. To this end, in
this work we make the following contributions:
\begin{itemize}
    \item {\bf Unifying Clustering Framework.} We present a framework
        encapsulating a series of existing clustering methods able to model
        data as complex as multi-view graphs with multiple view structures.
    \item {\bf A Principled Multi-View Graph Clustering Method.} We propose
        \method{}, a novel multi-view clustering method for graphs with
        multiple view structures, which can be seen as both a special case
        of the aforementioned framework and a generalization of the celebrated spectral
        clustering family of methods. We also show that it is very closely
        connected to k-means \cite{lloyd1982least} as well. 
    \item {\bf In-Depth Experimentation\footnote{An implementation of \method{}, along with
        the experiment scripts, is available at
        \url{https://github.com/gtsitsik/genclus}}.} We design a series of
        experiments
        on both artificial and real-world data to assess the performance of
        \method{} both quantitatively and qualitatively. In-depth
        comparisons with other baselines are performed as well.
\end{itemize}

\section{Proposed Unifying Framework.}

\begin{table}
    \centering
    \footnotesize
    \begin{tabular}{ |c|l|  }
        \hline
        \textbf{Symbol} & \textbf{Definition}\\
        \hline
        $x$,$X$                         &   Scalar  \\ 
        $\mat{x}$                       &   Column vector\\  
        $\mat{X}$                       &   Matrix\\
        $\tensor{X}$                    &   Tensor\\
        $\myset{X}$                     &   Set\\ 
        $\norm{\cdot}$                  &   Frobenious norm\\
        $\nnz{\cdot}$                   &   Number of non-zero elements\\
        $\mat{I}$                       &   Identity matrix\\
        $\mat{0}$                       &   All zeros column vector/matrix/tensor\\
        $\mat{1}$                       &   All ones column vector/matrix/tensor\\
        $\diag{\mat{x}}$                &   Diagonal matrix with $\mat{x}$ as its diagonal\\
        $\diag{\mat{X}}$                &   Vector consisting of the diagonal elements of $\mat{X}$\\
        $\mat{D}_\mat{X}^{(k)}$         &   $\diag{\mat{X}_{k:}}$\\
        $\tr{\mat{X}}$                  &   Trace of $\mat{X}$\\
        $\mat{X}^T$                     &   Transpose of $\mat{X}$\\
        $\mat{x}\succeq\mat{y}$         &   $\forall i \ \mat{x}_i \geq \mat{y}_i$ \\
        $\cp{U}{V}{W}{}{}{}$            &   PARAFAC with factor matrices $\mat{U}$, $\mat{V}$ and $\mat{W}$\\
        $[\mat{U}, \mat{V}, \mat{W}]$   &   Horizontal concatenation of $\mat{U}$, $\mat{V}$ and $\mat{W}$\\
        $\odot$                         &   Column-wise Khatri–Rao product\\
        $\times_n$                      &   Mode-$n$ product\\
        $\mat{X}_{(n)}$                 &   Mode-$n$ matricization of  $\tensor{X}$ \reminder{mention type of matriciation}\\
        \hline
    \end{tabular}
    \caption{Table of Symbols. See \cite{kolda2006multilinear} for further details on tensor
    notation and operations.}
    \label{tb:symbols}
\end{table}

\newlength{\myheight}
\begin{figure*}
\centering
\footnotesize
\begin{forest}
baseline,
for tree=
    {grow'=0,
    parent anchor=east,
    child anchor=west,
    align=center,
    edge={->},
    l sep =55
    },
forked edges,
  arrow just/.style={
    tikz+={
      \node [align=center] at ([yshift=30pt].mid |- e) {#1};
    },
  },
  tikz+={\coordinate (e) at (current bounding box.north);}
[$\cp{U}{V}{AB}{}{}{}$\\{\scriptsize(proposed}\\{\scriptsize
    unifying}\\{\scriptsize framework)},before drawing tree={y+=15.5}
    [\textbf{ComClus} \cite{ni2017comclus},arrow
    just={\settoheight{\myheight}{\includegraphics[width=\textwidth/8]{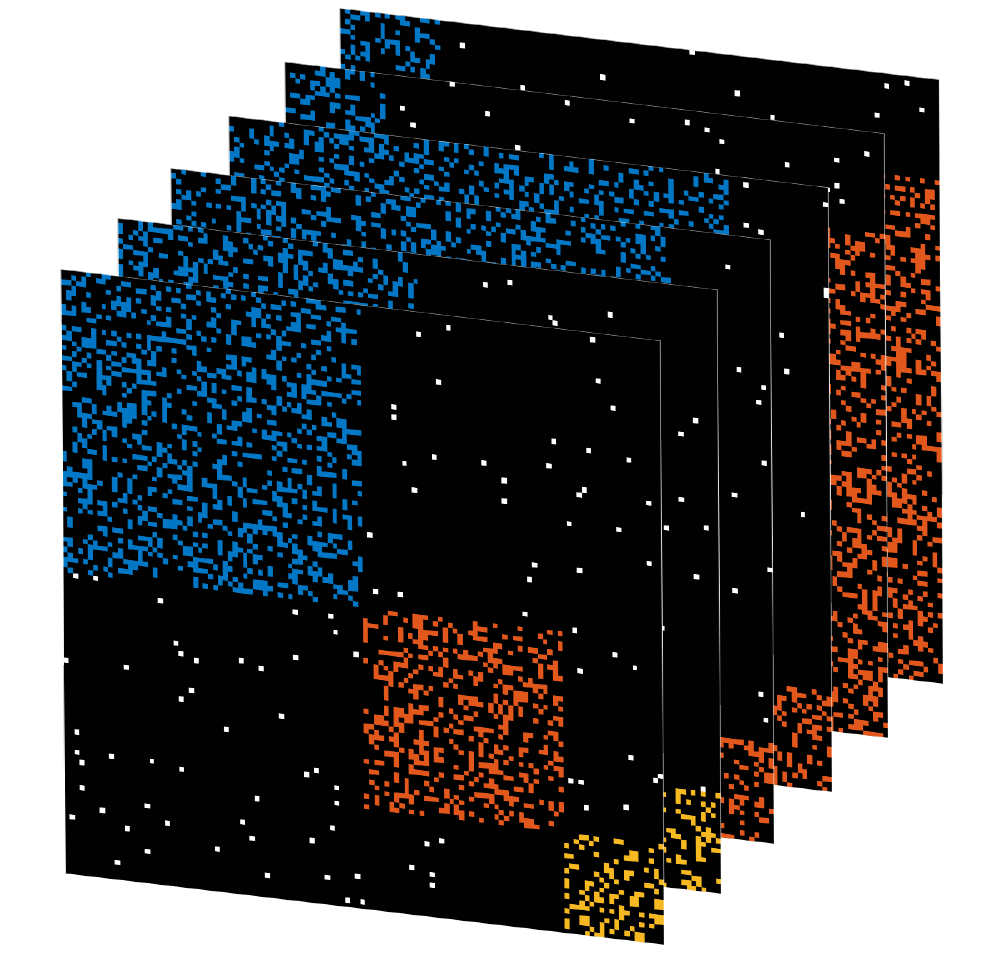}}\includegraphics[width=\textwidth/8]{Images/multi_structure_multi_relational_graph_2.pdf}\raisebox{0.4\myheight}{\fbox{\makecell{Multi-Structure
    \\ Multi-View}}}\hspace{\textwidth/8}},before drawing tree={y-=5}]
    [\textbf{\method{}}\\{\scriptsize(proposed method)}, 
    name=GC
        [\textbf{MC-TR-I-EVD} \cite{liu2012multiview},before drawing tree={y-=19},arrow
        just={\settoheight{\myheight}{\includegraphics[clip,trim=-50pt 0
        -50pt
        -53pt,width=\textwidth/9]{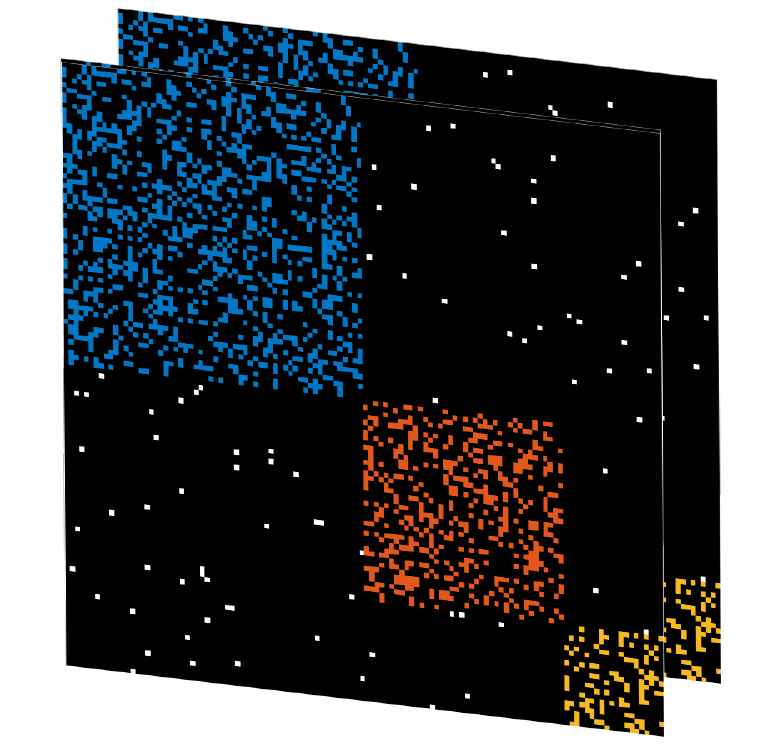}}\includegraphics[clip,trim=-50pt
        0 -50pt
        -53pt,width=\textwidth/9]{Images/multi_relational_graph_2.pdf}\raisebox{0.4\myheight}{\fbox{Multi-View}}\hspace{\textwidth/9}}, name=MSC1
            [\textbf{Spectral Clustering} \cite{ng2001spectral},before
            drawing tree={y-=41.2},no edge,arrow
            just={\settoheight{\myheight}{\includegraphics[clip,trim=-62pt 0 -62pt
            -67,width=\textwidth/9]{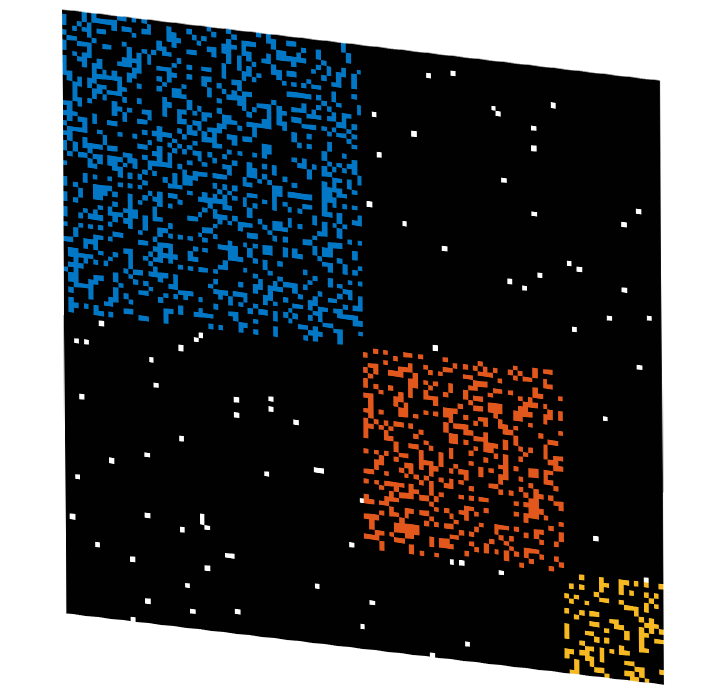}}\includegraphics[clip,trim=-62pt 0 -62pt
            -67,width=\textwidth/9]{Images/single_view_graph.pdf}\raisebox{0.4\myheight}{\fbox{Single-View}}\hspace{\textwidth/9}} ,name=SC,tier=singleview]
        ]
        [\textbf{MC-TR-I-EVDIT} \reminder{make it child of MC-TR-I-EVD} \cite{liu2012multiview},before drawing tree={y-=19.5},
        name=MSC2
            [,no edge,tier=singleview]
        ]
        [\textbf{\parbox{3cm}{\centering Co-regularized\\Spectral Clustering}} \cite{kumar2011co},before drawing tree={y-=16}, name=MSC3
        ]
        [\phantom{},no edge]
        [\textbf{k-means} \cite{lloyd1982least},before drawing tree={y-=15},tier=singleview] 
    ] 
    [\textbf{Richcom} \cite{gujral2020beyond},before drawing tree={y+=5}]
        [\textbf{CMNC} \cite{chen2019tensor},before drawing tree={y+=5}]
]
\draw[->] (MSC3.east) -- ($(SC.south west)!0.3!(SC.north west)$);
\draw[->] (MSC2.east) -- ($(SC.south west)!0.5!(SC.north west)$);
\draw[->] (MSC1.east) -- ($(SC.south west)!0.7!(SC.north west)$);
\end{forest}

\caption{For each arrow the method on its right can be seen as a special case of the method on its left. }
\label{fig:methodshierarchy}
\end{figure*}

Consider $K$ adjacency matrices, $\{\mat{X}^{(k)}\}_{k=1}^K$, where each is
of size $I \times I$ and  corresponds to a different view of a multi-view
graph. Additionally, assume that these views can be clustered into $M$ groups,
with each group corresponding to a different clustering of nodes. For
convenience, we will refer to such a node clustering structure as a
view structure. Given these, in this section we present an overview of
existing methods designed to model and extract such clusterings. Due to
space limitations, we will only discuss the embedding generation stage of
these methods, but a more thorough discussion on obtaining the final clusterings
is available in \autoref{sec:prob_form}. Then, we demonstrate that all
these methods can be expressed as special cases of a unifying framework, as
illustrated in \autoref{fig:methodshierarchy}.  Specifically, we show that
they can be expressed as variants of PARAFAC \cite{carroll1970analysis},
\cite{PARAFAC}, where the third factor matrix is constrained to be the
product of two matrices.

Please note that the goal of this unifying framework is not necessarily to
precisely encapsulate the exact optimization problems of these methods.
Rather, its goal is to abstract their essence. For example, we will assume
that a method is conceptually aiming to calculate the same model, independent
of whether it algorithmically imposes a constraint as a hard constraint or
as a soft constraint. Similarly, we will assume that it conceptually remains the
same, independent of whether it imposes constraints like non-negativity or
sparsity.

\subsection{ComClus.} 
ComClus \cite{ni2017comclus} can be expressed as  
\begin{multline}
    \inf_{\mat{U},\mat{W},\mat{A},\mat{B}\geq 0} \sum_{k=1}^K
    \norm{\mat{X}^{(k)}- \mat{O}^{(k)}\mat{U} \mat{D}_\mat{W}^{(k)}
    (\mat{O}^{(k)}\mat{U})^T }^2+
    \\
    r(\mat{U},\mat{W},\mat{A},\mat{B})
    \label{eq:comclus}
\end{multline}
where
$r(\mat{U},\mat{W},\mat{A},\mat{B}):=\beta||\mat{W}-\mat{A}\mat{B}||^2+\rho(||\mat{U}||_1+||\mat{A}||_1+||\mat{B}||_1)$,
each $\mat{O}^{(k)}$ is a user-defined indicator matrix, and $\beta$ and
$\rho$ are user-defined penalty parameters. By making the closed-world
assumption \cite{nickel2015review}, the first term becomes equal to 
\begin{equation*}
    \sum_{k=1}^K \norm{\mat{X}^{(k)}- \mat{U} \mat{D}_\mat{W}^{(k)} \mat{U}^T }
    =
     \norm{\tensor{X}-\cp{\mat{U}}{\mat{U}}{\mat{W}}{}{}{}}
\end{equation*}
where $\tensor{X}_{::k}:=\mat{X}^{(k)}$. Therefore, ComClus can be
interpreted as aiming to approximate $\tensor{X}$ by $\cp{U}{U}{W}{}{}{}$ such that $\mat{W}\approx \mat{A}\mat{B}$, where $\mat{U},\mat{A},\mat{B}$ are sparse and non-negative, and $\mat{W}$ is non-negative.

\subsection{Richcom.}
Richcom \cite{gujral2020beyond} can be expressed as 
\begin{multline}
    \inf_{\substack{\mat{U}^{(m)},\mat{V}^{(m)},\\ \mat{a}^{(m)} }}
    \norm{\tensor{X}-\sum_{m=1}^M \left(\mat{U}^{(m)}
    {\mat{V}^{(m)}}^T\right)\times_3 \mat{a}^{(m)}}^2
    +\\
    \sum_{m=1}^M r\left(\mat{U}^{(m)},\mat{V}^{(m)},\mat{a}^{(m)}\right)
    \label{eq:richcom}
\end{multline}
where $r\left(\mat{U}^{(m)},\mat{V}^{(m)},\mat{a}^{(m)}\right)$ encodes the
sparsity and non-negativity constraints. We notice that
\begin{equation*}
    \begin{split}
        \sum_{m=1}^M \left(\mat{U}^{(m)} {\mat{V}^{(m)}}^T\right)\times_3
        \mat{a}^{(m)} &= 
        \sum_{m=1}^M \cp{U}{V}{a}{^{(m)}}{^{(m)}}{^{(m)}\mat{1}^T}\\
        &=\cp{U}{V}{AB}{}{}{}
    \end{split}
\end{equation*}
where $\mat{U}:=
\begin{bmatrix}
    \mat{U}^{(1)}, \cdots, \mat{U}^{(M)}
\end{bmatrix}$, 
$\mat{V}:=
\begin{bmatrix}
    \mat{V}^{(1)}, \cdots,\mat{V}^{(M)} 
\end{bmatrix}$,
$\mat{A}:=
\begin{bmatrix}
    \mat{a}^{(1)}, \cdots, \mat{a}^{(M)}
\end{bmatrix}$ 
and $\mat{B}$ is an appropriate indicator matrix. The constraints of Richcom
are then equivalent to requiring $\mat{U}$,$\mat{V}$, and $\mat{A}$ to be
sparse and non-negative. Notice that, as opposed to ComClus, $\mat{B}$ here
is a fixed matrix that needs to be defined by the user. \reminder{mention
that these models are also a special case of block term decomposition}

\subsection{Centroid-based Multilayer  Network Clustering  (CMNC).} CMNC
\cite{chen2019tensor} is similar to Richcom with their
differences being that for CMNC $\mat{U}^{(m)}=\mat{V}^{(m)}$ for all
$m$,  $\sum_{m=1}^M \mat{a}^{(m)}_i=1$ for all $i$, no sparsity is
imposed and each view of the data tensor is preprocessed as $
{\mat{D}^{(k)}}^{-\frac{1}{2}}\mat{X}^{(k)}{\mat{D}^{(k)}}^{-\frac{1}{2}}$
where $\mat{D}^{(k)}:=\diag{\sum_{i=1}^I\mat{X}^{(k)}_{:i}}$.

\subsection{Spectral Clustering.}
\label{sec:spec_clus}

In spectral clustering \cite{von2007tutorial}, \cite{ng2001spectral},
\cite{zhou2005learning}, \cite{klus2022koopman}, an important quantity is the
Laplacian, $\mat{L}$,  defined as $\mat{D}-\mat{X}$, with
$\mat{D}:=\diag{\sum_{i=1}^I\mat{X}_{:i}}$. In turn, the symmetric
normalized Laplacian, $\mat{L}_{sym}$, is defined as
${\mat{D}}^{-\frac{1}{2}}\mat{L}{\mat{D}}^{-\frac{1}{2}}$.  Note that the
eigenspace of $\mat{L}_{sym}$ corresponding to its eigenvalue $\lambda \in
[0,2]$ is also the eigenspace of $\mat{S} := \mat{I}-\mat{L}_{sym} =
{\mat{D}}^{-\frac{1}{2}}\mat{X}{\mat{D}}^{-\frac{1}{2}}$ corresponding to
its eigenvalue $1-\lambda \in [-1,1]$.

\subsubsection{Single-View Spectral Clustering.}

Single-view spectral clustering can be formulated as
$\inf_{\mat{U}^T\mat{U}=\mat{I}}\norm{-\mat{L}_{sym}-\mat{U}\mat{U}^T}$, and
notice that
$\mat{U}\mat{U}^T=\cp{U}{U}{\mat{1}^T}{}{}{}=\cp{U}{U}{\mat{A}\mat{B}}{}{}{}
$, where $\mat{A}:=1$ and $\mat{B}:= \mat{1}^T$.

\subsubsection{Multi-View Spectral Clustering.}

In \cite{liu2012multiview} the authors proposed the MC-TR-I-EVDIT method
which models multi-view spectral clustering as  
\begin{equation}
    \label{eq:multi_weighted}
    \sup_{\substack{\mat{U}^T\mat{U}=\mat{I}\\\mat{a}\succeq\mat{0}, \norm{\mat{a}}=1
    }} \sum_{k=1}^K\tr{\mat{U}^T\mat{a}_k\mat{S}^{(k)}\mat{U}}.
\end{equation}
They also proposed the MC-TR-I-EVD method which can be seen as a variant of
MC-TR-I-EVDIT where all elements of $\mat{a}$ are further constrained to
be identical.

Then, in \cite{kumar2011co} the authors propose two forms of  co-regularized
spectral clustering with one utilizing a centroid-based co-regularization as
\begin{multline}
    \label{eq:centroid_coreg}
    \sup_{\substack{{\mat{U}^{(k)}}^T\mat{U}^{(k)}=\mat{I}\\
    {\mat{U}^*}^T\mat{U}^*=\mat{I}}}
    \sum_{k=1}^K\tr{{\mat{U}^{(k)}}^T\mat{S}^{(k)}\mat{U}^{(k)}}
    +
    \\
    \lambda_k \tr{\mat{U}^{(k)}{\mat{U}^{(k)}}^T\mat{U}^*{\mat{U}^*}^T}
\end{multline}
and the other utilizing the pairwise co-regularization $\lambda
\sum_{i\neq
j}\tr{\mat{U}^{(i)}{\mat{U}^{(i)}}^T\mat{U}^{(j)}{\mat{U}^{(j)}}^T}$
instead.  Note that both of these formulations can be seen as relaxed
versions of MC-TR-I-EVD where the embeddings from different views are
forced to only be similar to each other instead of exactly equal.
Specifically, the regularization terms force the columnspaces of all
$\mat{U}^{(k)}$ to become more similar with each other as $\lambda$ and
$\lambda_k$ take larger values. In the limit, their columnspaces become
identical and the  rows of $\mat{U}^{(k)}$ will be just orthogonally
rotated versions of the rows of any other  $\mat{U}^{(i)}$.  Therefore,
applying k-means, as suggested by the authors,
on the rows of any $\mat{U}^{(k)}$ will lead to the same node clusters.

Now notice that \eqref{eq:multi_weighted} has the same optimal set as
\begin{equation}
 \arginf_{
     \substack{\mat{U}^T\mat{U}=\mat{I}\\\mat{A}\succeq\mat{0}, \norm{\mat{A}}=1}
 }
    \norm{\tensor{S}-\cp{U}{U}{AB}{}{}{}}
    \label{eq:multi_view_spectral}
\end{equation}
where $\tensor{S}_{::k}:=\mat{S}^{(k)}$, $\mat{A}:=\mat{a}$ and
$\mat{B}:=\mat{1}^T$ (see \autoref{sec:multi_view_spectral_proof} for
proof). Similarly, we can show that MC-TR-I-EVD has the same optimal set as
$\inf_{\mat{U}^T\mat{U}=\mat{I} }\norm{\tensor{S}-\cp{U}{U}{AB}{}{}{}}$
where $\tensor{S}_{::k}:=\mat{S}^{(k)}$, $\mat{A}:=\mat{1}$ and
$\mat{B}:=\mat{1}^T$. Lastly, note that due to the close connection of
\eqref{eq:centroid_coreg} and its variant with MC-TR-I-EVD, one can argue
that the same reformulation captures the essence of these problems as well.

\section{Proposed Method.}
\label{sec:MSSC}
In this section, we develop \method, a principled graph clustering method
which can be viewed as an instance of this unifying framework. \method{} also
generalizes spectral clustering \cite{von2007tutorial} to
multi-view graphs with multiple view structures, and, as we will show in
\autoref{sec:interpret}, it is closely associated with k-means as well.

Based on our discussion in \autoref{sec:spec_clus}, we will consider the
modified data tensor $\tensor{Y}$, where
$\tensor{Y}_{::k}:={\mat{D}^{(k)}}\tensor{X}_{::k}{\mat{D}^{(k)}}$ and
$\mat{D}^{(k)}:=\diag{ 1/\sqrt{\sum_{i=1}^I\mat{X}_{:ik}}}$.
Then, we observe that in the special case of a single-view graph, where $K=M=1$ and
$R$ is the true number of clusters, we can use
properties of the best positive semi-definite approximation of a matrix (see
\autoref{sec:psd_approx_proof} for details), along with the fact that we expect
$\tensor{Y}$ to have $R$ eigenvalues close to 1, to show that
\begin{multline}
    \arginf_{{\mat{U}^T\mat{U}=\mat{I}}}
    \left(\inf_{a\geq0,\mat{b}\succeq\mat{0}}\norm{\tensor{Y}-a\mat{U}\diag{\mat{b}}\mat{U}^T}\right)=
    \\
    \arginf_{{\mat{U}^T\mat{U}=\mat{I}}}\norm{\tensor{Y}-\mat{U}\mat{U}^T}.
    \label{eq:single_view}
\end{multline}
Notice that the constraints on $a$ and $\mat{b}$ of the left-hand side
problem are more relaxed than the constraints on the implicit $a$ and
$\mat{b}$ of the right-hand side problem. Therefore, this result can be
particularly useful in designing solvers that are less prone to bad local
optima. Also, to understand why we opted for non-negativity constraints,
first note that allowing $a$ to be negative may make \eqref{eq:single_view}
to not hold for any arbitrary $\tensor{Y}$ since the left-hand side problem
will select eigenvectors corresponding to its largest negative eigenvalues
if they are of larger magnitude than its positive eigenvalues. This is in
contrast to the right-hand side problem which always retrieves eigenvectors
corresponding to the maximum eigenvalues. Similarly, if we allow $\mat{b}$
to have negative elements, then the left-hand side problem is a direct
application of the Eckart-Young theorem \cite{eckart1936approximation} and
will, therefore, select eigenvectors corresponding to the maximum magnitude
eigenvalues instead.  

To design a generalized version of this problem for arbitrary values of $K$
and $M$, first notice that $
a\mat{U}\diag{\mat{b}}\mat{U}^T=\cp{U}{U}{}{}{}{a\mat{b}^T}=\cp{U}{U}{AB}{}{}{}
$, where $\mat{A}:= a$ and $\mat{B}:= \mat{b}^T$. Therefore, in the general
case where $\tensor{Y}$ has multiple views with multiple view structures, we
propose formulating the problem as
\begin{equation}
    \inf_{\substack{{\mat{U}^{(m)}}^T\mat{U}^{(m)}=\mat{I}
    \\
    \mat{A},\mat{B}^T\in \myset{I}}} \norm{\tensor{Y}-\cp{U}{U}{AB}{}{}{}}^2
    \label{eq:proposed1}
\end{equation}
where $\myset{I}$ is the set of all matrices whose rows contain only a
single non-zero element each, and the columns of $\mat{U}^{(m)}$ are defined
to be the subset of columns of $\mat{U}$ that correspond to the positions of
the non-zero elements of $\mat{B}_{m:}$. Note that the constraint
$\mat{A},\mat{B}^T\in \myset{I}$ is particularly important for two reasons.
First, it leads to a $\left\{\mat{U}^{(m)}\right\}_{m=1}^M$ that forms a
partition of the columns of $\mat{U}$, which in turn implies that we get a
distinct model for each view structure. Second, as we will show next, it
enables $\mat{U}$ and $\mat{B}$ to be calculated simultaneously, which can
prove beneficial in terms of designing a solver that is both faster and less
prone to bad local optima.

Lastly, note that for completeness, in \autoref{sec:optim_steps} where we
will derive optimization steps for \eqref{eq:proposed1}, we will also
derive optimization steps for further constrained versions of it where the
non-zero elements of $\mat{A}$ and $\mat{B}$ can be either all-ones or
non-negative. That said, we still recommend using the non-negativity
constraint for both $\mat{A}$ and $\mat{B}$ as the default option. To
understand the reasoning behind this choice, first notice that
\eqref{eq:proposed1} can properly generalize the left-hand side and
right-hand side of \eqref{eq:single_view}, only when the non-negativity and
all-ones constraint is imposed, respectively. That is, without these
additional constraints, \eqref{eq:proposed1} does not necessarily lead to
generalization of spectral clustering. Additionally, notice that the
non-negativity constraint may be preferable to the all-ones constraint
because it leads to an optimization problem that is more relaxed, and,
therefore, the resulting solver can be expected to be less prone to bad
local optima. 

\subsection{Optimization Steps.}
\label{sec:optim_steps}
We propose solving \eqref{eq:proposed1} in a block coordinate descent
fashion by alternatingly updating $\mat{A}$, and then $\mat{U}$ and
$\mat{B}$ simultaneously. Note that, although in this work we do not provide
arguments regarding convergence, our optimization scheme is guaranteed to
monotonically improve the objective function after each update of $\mat{A}$,
$\mat{B}$ and $\mat{U}$.

\subsubsection{Steps for $\mat{A}$.}
By observing that $\norm{\tensor{Y}-\cp{U}{U}{AB}{}{}{}} =
\norm{\mat{Y}_{(3)}-\mat{A}\mat{B}\left(\mat{U} \odot \mat{U}\right)^T}$, we
can see that for fixed $\mat{B}$ and $\mat{U}$, the optimal $\mat{A}$ is
given by the solution of a constrained linear least squares problem.
Specifically, below we discuss three different constraints for the non-zero
elements of $\mat{A}$:

\paragraph{All-ones $\mat{A}$.}
Here we observe that the $k$-th row of $\mat{A}$ is optimal when it is an
indicator vector with its $m$-th element equal to 1 where $m =
\arginf_m\norm{\left[\mat{Y}_{(3)}\right]_{k:}-\left[\mat{B}\left(\mat{U}
\odot \mat{U}\right)^T\right]_{m:}}$, or equivalently by using tensor
notation
$m=\arginf_{m}\norm{\tensor{Y}_{::k}-\mat{U}\diag{\mat{B}_{m:}}\mat{U}^T}$. 

\paragraph{Unconstrained $\mat{A}$.}
Here, the only non-zero element of $\mat{A}_{k:}$ will be at position $m$
if, and only if, among all lines defined by each row of
$\mat{B}\left(\mat{U}\odot\mat{U}\right)^T$, the one defined by the $m$-th
row is the closest one to the $k$-th row of $\mat{Y}_{(3)}$. Thus,
$m=\argsup_m
\left|\sum_{i,j}\tensor{Y}_{ijk}\mat{Q}^{(m)}_{ij}\right|/\norm{\mat{Q}^{(m)}}$
where $\mat{Q}^{(m)}:=\mat{U}\diag{\mat{B}_{m:}}\mat{U}^T$. In turn, we have
$\mat{A}_{km}=\arginf_{\alpha}\norm{\tensor{Y}_{::k}-\alpha \mat{Q}^{(m)}}$
which can be calculated in closed form as
$\sum_{i,j}\tensor{Y}_{ijk}\mat{Q}^{(m)}_{ij}/\norm{\mat{Q}^{(m)}}^2$.

\paragraph{Non-negative $\mat{A}$.}
Here, the only non-zero element of $\mat{A}_{k:}$ can be $\mat{A}_{km}$ if,
and only if, $\left[\mat{Y}_{(3)}\right]_{k:}$ forms the smallest angle with
the $m$-th row of $\mat{B}\left(\mat{U}\odot\mat{U}\right)^T$. Thus, we have
$m=\argsup_m
\sum_{i,j}\tensor{Y}_{ijk}\mat{Q}^{(m)}_{ij}/\norm{\mat{Q}^{(m)}}$  where
$\mat{Q}^{(m)}:=\mat{U}\diag{\mat{B}_{m:}}\mat{U}^T$. In turn,
$\mat{A}_{km}$ is calculated as in the unconstrained case, unless
$\sum_{i,j}\tensor{Y}_{ijk}\mat{Q}^{(m)}_{ij}$ is negative, in which case
$\mat{A}_{km}=0$.  Note that $\mat{A}_{km}=0$ if, and only if, none of the
rows of $\mat{B}\left(\mat{U}\odot\mat{U}\right)^T$ forms an acute angle
with $\left[\mat{Y}_{(3)}\right]_{k:}$.

\subsubsection{Steps for $\mat{U}$ and $\mat{B}$.}
Here, first notice that \eqref{eq:proposed1} can be reexpressed as
\begin{equation}
\arginf_{{\substack{{\mat{U}^{(m)}}^T\mat{U}^{(m)}=\mat{I}
    \\
    \mat{A},\mat{B}^T\in \myset{I}}}} \sum_{m=1}^M
 \sum_{k=1}^K  \norm{\tensor{Y}_{::k}-\mat{A}_{km}\mat{U}^{(m)}\mat{D}_{\mat{B}}^{(m)}{\mat{U}^{(m)}}^T}^2 \label{eq:U_B_update}
\end{equation} 
where  $\mat{D}_{\mat{B}}^{(m)}$ is defined to be a diagonal matrix
containing only the non-zero elements of $\mat{B}_{m:}$. We also define
$\mathcal{S}_m$ as the set of indices of the views assigned to view cluster $m$.
In turn, for a fixed $\mat{A}$\reminder{explain what happens when a column
of A is all-zeros}, we distinguish three types of constraints for the
non-zero elements of $\mat{B}$:

\paragraph{All-ones $\mat{B}$.} 
In this case, we can show that for a fixed $\mat{A}$ \eqref{eq:U_B_update} can be reexpressed as
\begin{equation}
    \argsup_{\substack{{\mat{U}^{(m)}}^T\mat{U}^{(m)}=\mat{I}\\ \mat{B}^T\in \mathcal{I}}}
    \sum_{m=1}^M
    \tr{{\mat{U}^{(m)}}^T \mat{Z}^{(m)}\mat{U}^{(m)}}
    \label{eq:proposed1_2}
\end{equation}
where $\mat{Z}^{(m)}:=\sum_{k=1}^K
2\mat{A}_{km}\tensor{Y}_{::k}-\norm{\mat{A}_{:m}}^2\mat{I}$ 
(see \autoref{sec:proposed1_2_proof} for detailed derivation). Therefore, we
can see that, for a fixed $\mat{B}$ and for all $m$, the optimal
$\mat{U}^{(m)}$ has columns the eigenvectors of $\mat{Z}^{(m)}$
corresponding to its largest eigenvalues\reminder{cite}, which in turn
implies that the corresponding summand in the objective function of
\eqref{eq:proposed1_2} will be the sum of these eigenvalues.  Thus, if we
define $\mathcal{E}_m$ to be a set containing the eigenvalues of
$\mat{Z}^{(m)}$, and $\mathcal{E}_{max}$ to be a set containing the largest
$R$ elements of $\cup_{m=1}^M \mathcal{E}_m$, then we can see that the
optimal value of the objective function in \eqref{eq:proposed1_2} cannot be
greater than the sum of the elements of $\mathcal{E}_{max}$. In fact, we can
achieve this value if for all $m$ we set $\mat{U}^{(m)}$ to have columns the
eigenvectors of $\mat{Z}^{(m)}$ corresponding to the eigenvalues in
$\mathcal{E}_m \cap \mathcal{E}_{max}$. Then, the optimal pair
$(\mat{U},\mat{B})$ can be given by setting
$\mat{U}=\begin{bmatrix}\mat{U}^{(1)}, \mat{U}^{(2)}, \cdots, \mat{U}^{(M)}
\end{bmatrix}$ and then deriving the optimal $\mat{B}$ by noticing that
$\mat{B}_{mr}$ is non-zero if, and only if, any of the columns of
$\mat{U}^{(m)}$ was assigned as the $r$-th column of $\mat{U}$.
\reminder{what happens when multiple eigenvalues are identical}

\paragraph{Unconstrained $\mat{B}$.}
In this case, for a fixed $\mat{A}$, \eqref{eq:U_B_update} can be
reexpressed as
\begin{equation}
     \arginf_{\substack{{\mat{U}^{(m)}}^T\mat{U}^{(m)}=\mat{I}
    \\ \mat{B}^T\in \mathcal{I}}}
     \sum_{m=1}^M
\norm{\mat{Z}^{(m)}-\norm{\mat{A}_{:m}}\mat{U}^{(m)}\mat{D}_{\mat{B}}^{(m)}{\mat{U}^{(m)}}^T}^2
\label{eq:proposed2}
\end{equation}
where
$\mat{Z}^{(m)}:=\sum_{k=1}^K\mat{A}_{km}\tensor{Y}_{::k}/\norm{\mat{A}_{:m}}$
(see \autoref{sec:proposed2_proof} for detailed derivation). 
We can now see that if we leave the non-zero elements of $\mat{B}$
unconstrained, then, for all $m$,  $\mat{U}^{(m)}$ will have a fixed number
of columns which will be optimal when they consist of the eigenvectors of
$\mat{Z}^{(m)}$ corresponding to the eigenvalues of largest magnitude. In
turn, the non-zero elements of the optimal $\mat{B}_{m:}$ are the same
eigenvalues divided by $\norm{\mat{A}_{:m}}$\reminder{cite}. Also, notice
that the corresponding summand in the objective function of
\eqref{eq:proposed2} will be the sum of squares of all the remaining
eigenvalues. Therefore, if we define $\mathcal{E}_m$ to be a set containing
the eigenvalues of $\mat{Z}^{(m)}$, and $\mathcal{E}_{max}$ to be a set
containing the $R$ elements of largest magnitude of $\cup_{m=1}^M
\mathcal{E}_m$, then we can see that the optimal value of the objective
function in \eqref{eq:proposed2} cannot be lower than the sum of the squared
elements of $\cup_{m=1}^M \mathcal{E}_m\setminus\mathcal{E}_{max}$. In fact,
we can achieve this value by setting the columns of $\mat{U}^{(m)}$ to be
the eigenvectors of $\mat{Z}^{(m)}$ corresponding to the eigenvalues in
$\mathcal{E}_m \cap \mathcal{E}_{max}$. Then, the optimal pair
$(\mat{U},\mat{B})$ can be given by setting
$\mat{U}=\begin{bmatrix}\mat{U}^{(1)}, \mat{U}^{(2)}, \cdots, \mat{U}^{(M)}
\end{bmatrix}$ and by assigning the elements of $\mathcal{E}_m \cap
\mathcal{E}_{max}$ divided by $\norm{\mat{A}_{:m}}$ as the non-zero elements
of $\mat{B}_{m:}$. 

\paragraph{Non-negative $\mat{B}$.} Here first notice that each
of the $M$ summands in \eqref{eq:proposed2} is minimized via the
best positive semi-definite approximation of the corresponding
$\mat{Z}^{(m)}$.
In other words, we can see that the same arguments as in the unconstrained
case apply, with the difference that $\mathcal{E}_m$ here is instead defined
to contain the largest eigenvalues of $\mat{Z}^{(m)}$ after its negative
eigenvalues are set equal to zero.

\subsection{Model Interpretation.}
\label{sec:interpret}
By looking at the updates for $\mat{U}$, $\mat{A}$ and $\mat{B}$ we can see
that there is a very natural way of interpreting \method{}. First, from
\eqref{eq:proposed1_2} and \eqref{eq:proposed2} we can see that, for a fixed
$\mat{A}$, \method{} computes the node clustering for each of the $M$ view
clusters. Specifically, the $m$-th node clustering is calculated, roughly
speaking, by performing spectral clustering based on $\mat{Z}^{(m)}$ which
is the weighted summation of the Laplacians of all views that
belong to the $m$-th view cluster. Also, for fixed $\mat{U}$ and $\mat{B}$,
and by noticing that $\cp{U}{U}{AB}{}{}{}=\cp{U}{U}{B}{}{}{}\times_3
\mat{A}$, we can see that each row of $\mat{A}$ is calculated in exactly the
same fashion as the cluster assignment step of k-means or a k-lines method,
depending on the type of constraint. Specifically, we can think of the
frontal slices of $\cp{U}{U}{B}{}{}{}$ as the centroids representing the
clusters, while the $k$-th row of $\mat{A}$ can be seen as an indicator
vector encoding the cluster membership of the $k$-th data point,
$\tensor{Y}_{::k}$. 

This observation leads us to another interesting finding. That is, if we
constrain the non-zero elements of $\mat{A}$ and $\mat{B}$ to be all-ones
and unconstrained, respectively, and we set $R=M\cdot I$, then \method{}
becomes identical to k-means. To see this, note that in this case, for a
fixed $\mat{A}$, there will always exist $\mat{U}$ and $\mat{B}$ in
\eqref{eq:proposed2} such that, for all $m$,
$\norm{\mat{A}_{:m}}\mat{U}^{(m)}\mat{D}_{\mat{B}}^{(m)}{\mat{U}^{(m)}}^T$
will be identical to $\mat{Z}^{(m)}$. Therefore, the
$m$-th frontal slice of the optimal $\cp{U}{U}{B}{}{}{}$ will be
$\mat{Z}^{(m)}/ \norm{\mat{A}_{:m}}$. In turn, this can be simplified to
$\sum_{k\in \mathcal{S}_m}\tensor{Y}_{::k}/\left|\mathcal{S}_m \right|$, with
$\mathcal{S}_m$ being the set of indices of the views assigned to the $m$-th
view cluster. In fact, this is exactly the centroid calculation step of
k-means. Also, notice that if we set $R<M\cdot I$, \method{} can be seen as
a version of k-means where instead of using the usual centroids, we perform
calculations based on their denoised low-rank versions
$\left\{\cp{U}{U}{B}{}{}{}_{::m}\right\}_{m=1}^M$.

\subsection{Space \& Time Complexity.}
\label{sec:space_time_complexity}

Note that our method can readily benefit from various optimizations that
leverage parallelization, sparsity, partial eigendecompositions and more
appropriate sorting techniques. However, for simplicity, we have omitted
such optimizations from the implementation used in all our experiments in
\autoref{sec:experiments}. In this case, a straightforward implementation
prioritizing time minimization  over memory usage can achieve a space
complexity of $O\left(MI^2+K\right)$, and, for $t$ iterations, a time
complexity of $O\left(tMI^3+tKMI^2+tMI\log(MI)\right)$. For more
details, please refer to \autoref{app:space_time_complexity}.

\section{Experiments.}
\label{sec:experiments}

In this section, we perform an in-depth experimental exploration of the
behavior of \method{} both quantitatively and qualitatively. First, in
\autoref{sec:artif_clust_perf} we perform quantitative clustering quality
comparisons with other baseline methods on artificially generated multi-view
graphs with known ground truth labels. Then, in \autoref{sec:real} we
present a qualitative case study which demonstrates the ability of \method{}
to generate meaningful clusterings on real-world multi-structure multi-view
graphs. Lastly, in \autoref{sec:artif_time} we perform an experimental
comparison of the time complexity of all methods. These experiments were
conducted using Multi-Graph Explorer \cite{multi-graph-explorer} whose code
can be accessed at \cite{multi-graph-explorer-github}.


Lastly, note that in all methods discussed so far, the clustering process can be
divided into four segments: data preprocessing, embedding calculation,
embedding postprocessing and embedding clustering.  However, embedding
calculation is arguably the central novelty in both the existing methods and
our proposed method. Therefore, we will consider both the clustering
schemes as proposed in the original papers and enhanced versions where the
best combination of the remaining segments is selected. We will call these
``original methods'' and ``enhanced methods'', respectively. For the precise
details of the experiment setup, please refer to
\autoref{sec:experiment_details}.

\reminder{
1) Explain how all-nan and all-zeros latent vectors are handled for clustering\\
2) Give example that shows why non-negative could be better than unconstrained (e.g. when overfactoring a  graph with bipartite communities)\\
3) Improve citation of code and datasets}

\subsection{Clustering Quality on Artificial Data.}
\label{sec:artif_clust_perf}

We generate a directed unweighted multi-view graph with 120 nodes and 9
views, which leads to a tensor of size $120\times 120 \times 9$. The views
form 3 clusters with 3 views each, and each view cluster corresponds to 3, 2
and 2 node clusters, respectively. Specifically, the node clusters
corresponding to the first view cluster contain 60, 40 and 20 nodes,
respectively, while in the second view cluster they contain 100 and 20
nodes, respectively, and in the third view cluster they contain 20 and 100
nodes, respectively. All node clusters are $\gamma$-quasi-cliques
\cite{pattillo2013maximum} whose intra-community edge density, $\gamma$,
takes values in the set of the 8 equally spaced values from $15\%$ to $1\%$,
inclusive.  For a specific generated graph, all quasi-cliques have identical
value of $\gamma$. Then, we randomly select $1\%$ of all pairs of nodes,
and, for each pair, we remove their connecting edge if they are connected,
or introduce a new edge between them if they are not connected. For more
details, please refer to \autoref{sec:artif_clust_perf_appendix}.

\subsubsection{Results Analysis.}
\label{sec:artif_cluster_perf:results}

\begin{figure}
    \centering
    \subcaptionbox*{}[\linewidth]{\includegraphics[clip, trim=0 149pt 0 0]{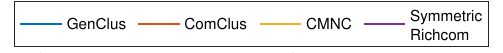}}
    \vspace*{-15pt}

    \subcaptionbox{Original methods\label{fig:clust_perf_orig}}[0.5\linewidth]{%
        \includegraphics[clip, trim=0 0 120.5pt 25pt ]{Images/artificial_performance_AMI_adobe_fixed.pdf}}%
    \subcaptionbox{Enhanced methods\label{fig:clust_perf_enhanced}}[0.5\linewidth]{%
        \includegraphics[clip, trim=120.5pt 0 0 25pt]{Images/artificial_performance_AMI_adobe_fixed.pdf}}%
    \caption{Clustering performance comparisons. Lines represent
    medians, while shaded areas represent 25-th and 75-th percentiles.
    Higher values signify better clustering quality and the maximum possible
    value is 1.}
    \label{fig:clust_perf}
\end{figure}

%
%

First, we make comparisons of the original methods as shown in
\autoref{fig:clust_perf_orig}. Here we see that \method{} offers superior
performance compared to all baselines, and, in fact, it is the only
method that manages, in the median, to perfectly reconstruct the
ground truth communities even with an intra-community edge density,
$\gamma$, as low as 11\%.  ComClus is the next best method, and it also
significantly outperforms CMNC and Symmetric Richcom. Also, although
ComClus slightly outperforms \method{} for very low values of $\gamma$, it
performs significantly worse than \method{} for higher values of $\gamma$.
CMNC outperforms Symmetric Richcom in the median, but due to its high
variance we cannot confidently declare it as the clear winner. 

Now, we make comparisons of the enhanced methods as shown in
\autoref{fig:clust_perf_enhanced}. In this case, we observe that the
enhanced versions of ComClus and Symmetric Richcom present a significant
performance uplift. In fact, ComClus now observably tends to perform better
than \method{}. At the same time, Symmetric Richcom not only became
significantly better than CMNC, but its performance is now closer to the
performance of \method. On the other hand, the enhanced versions of
\method{} and CMNC barely show a performance improvement.  However, note
that, at least for \method{}, this is a positive outcome, since it
experimentally validates our theoretical arguments from \autoref{sec:MSSC}
in favor of pairing \method{} with normalized Laplacians and with
non-negativity constraints for $\mat{A}$ and $\mat{B}$.

\subsection{Real-World Case Study.}
\label{sec:real}
\reminder{explain real dataset better\\
time it needed to run}
\reminder{give a brief description}

In this case study we are using a dataset of flight routes from 2012
\cite{openflights_dataset} containing flights from a large number of
airlines and airports around the world. Note that, to simplify the wording
of our observations, we will refer to the Americas as if they were a single
continent. For additional details, please refer to
\autoref{sec:real_world_1}.

\begin{figure}[!t]
    \centering
    \subfloat[Airline clusters]{
        \includegraphics[clip,trim=1cm 0cm 1cm 0cm,height=0.20\textwidth]{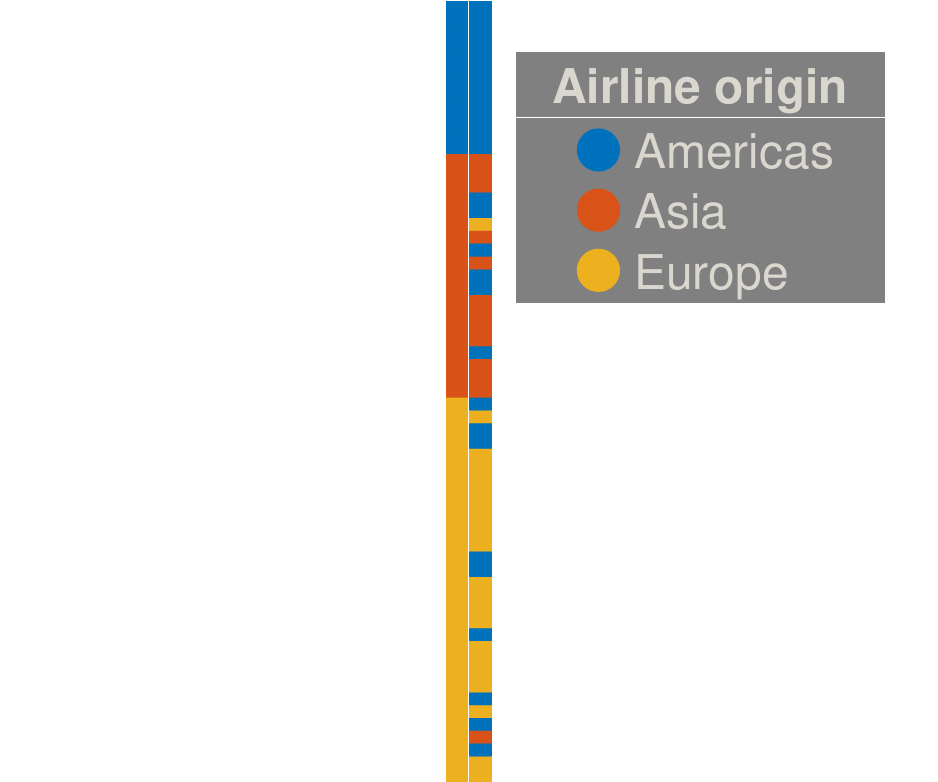}
        \label{fig:flights_views}
    }
    \subfloat[\raisebox{0pt}{\parbox[c]{0.15\textwidth}{\centering Airport clusters of Americas' airlines}}]{
        \includegraphics[clip,trim=1cm 0cm 1.0cm
        0cm,height=0.2\textwidth]{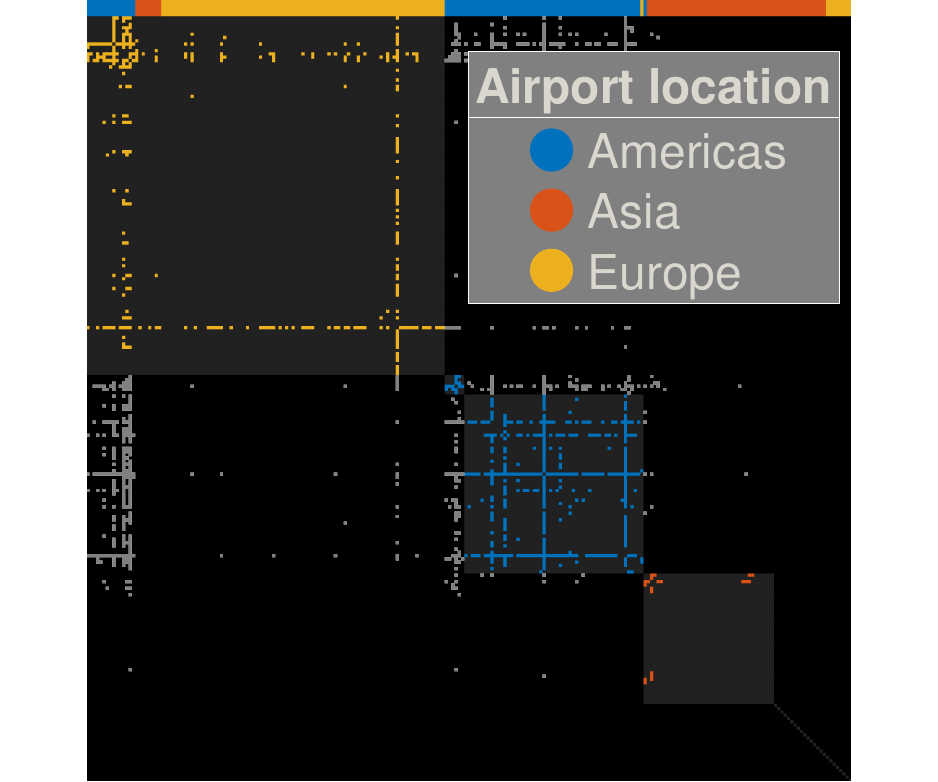}
        \label{fig:flights_nodes_1}
    }
\vspace{10pt}
    \subfloat[\raisebox{0pt}{\parbox[c]{0.15\textwidth}{\centering Airport clusters of Asian airlines}}]{
        \includegraphics[clip,trim=1.0cm 0cm 1cm
        0cm,height=0.2\textwidth]{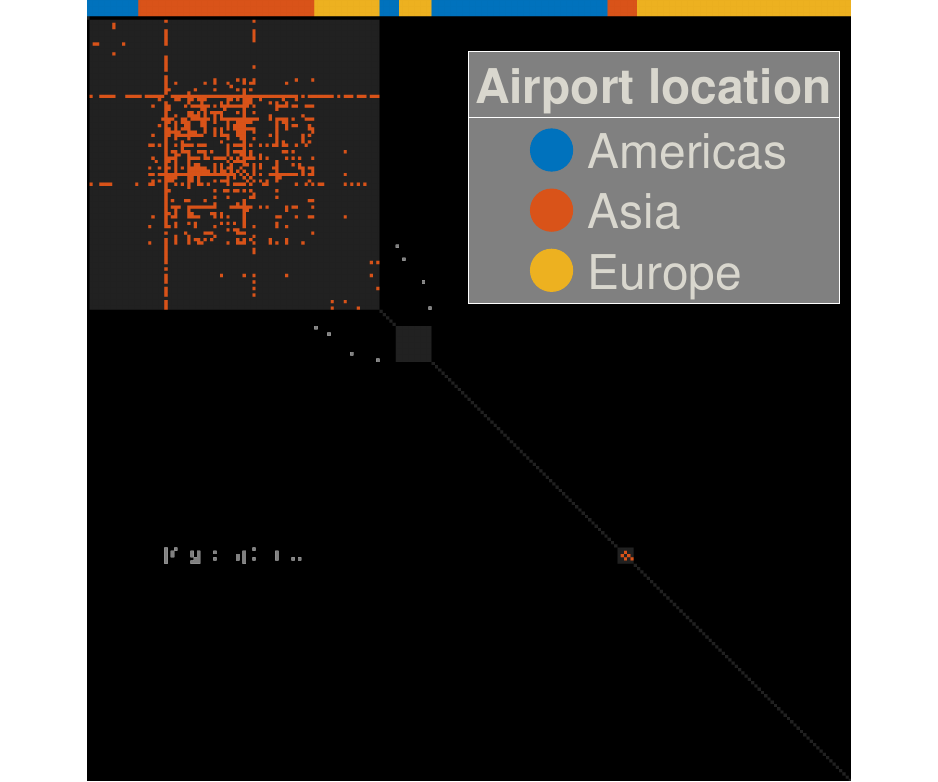}
        \label{fig:flights_nodes_2}
    }
    \subfloat[\raisebox{0pt}{\parbox[c]{0.15\textwidth}{\centering Airport clusters of European airlines}}]{
        \includegraphics[clip,trim=1.0cm 0cm 1cm
        0cm,height=0.2\textwidth]{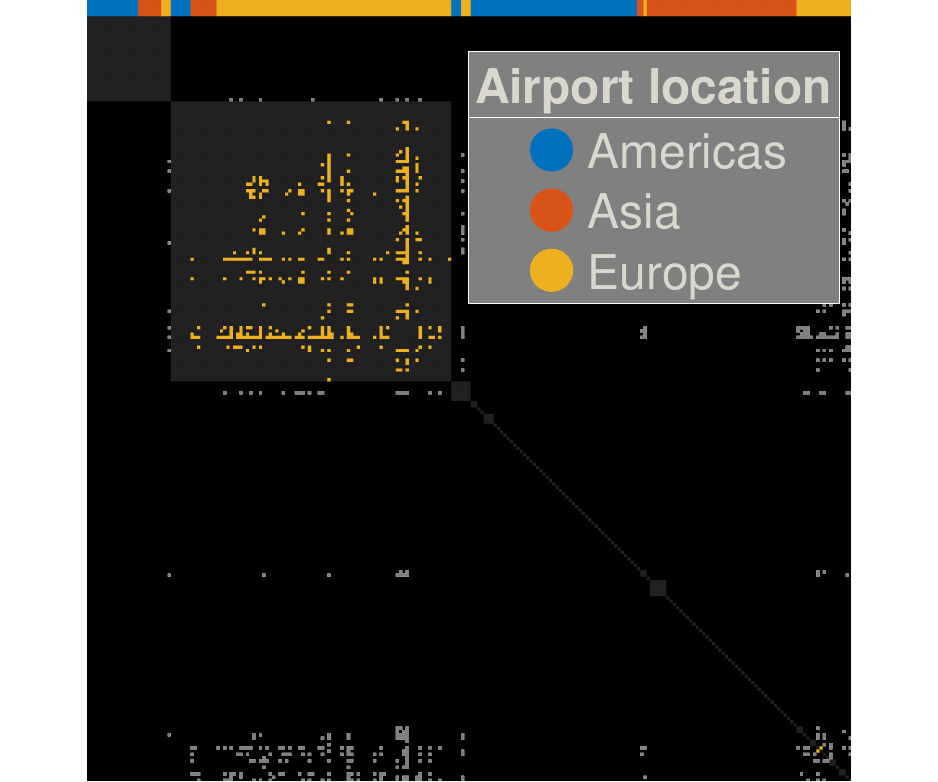}
        \label{fig:flights_nodes_3}
    }
    \caption{Clustering of airlines and airports by \method{} for the
    flights dataset. All colored bars are split into small pieces
    representing individual airports or airlines, and the colors represent
    labels. The left-hand bar in (a) depicts the airlines colored based on
    their continent of origin, while the right-hand bar shows them colored
    based on the view clustering produced by \method{}. (b)-(d) show the
    adjacency matrices of representative views from each of the airline
    clusters produced by \method, and the horizontal colored bars on top
    indicate the actual location of each airport. The airport clusters
    generated by \method{} are depicted as grey squares, which is achieved
    by appropriately permuting each adjacency matrix individually.
    }
    \label{fig:flights}
\end{figure}

\subsubsection{Results Analysis.}

\autoref{fig:flights} shows the results of our proposed method on this
dataset. As we can see, the clustering of views (airlines) generated by
\method{} separates the majority of them based on their continent of origin.
This aligns with our intuition that the flight patterns of airlines
originating in the same continent are similar to each other and different
from the flight patterns of airlines from different continents.
Additionally, \autoref{fig:flights_nodes_1} indicates that airlines from the
Americas tend to have substantial presence on all three continents, although
these flights tend to not be intercontinental.  On the other hand,
\autoref{fig:flights_nodes_2} and \autoref{fig:flights_nodes_3} indicate
that airlines originating in Asia and Europe tend to fly almost exclusively
within their continent of origin.

\subsection{Execution Time on Artificial Data.}
\label{sec:artif_time}

\begin{figure*}
    \centering
    \subcaptionbox*{}[\textwidth]{\includegraphics[clip, trim=54pt 122pt 54pt 0]{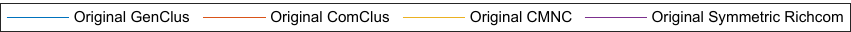}}
    \vspace*{-15pt}
    
    \newlength{\widthSmall}
    \setlength{\widthSmall}{0.2\textwidth} 
    \newlength{\widthBig}
    \setlength{\widthBig}{\dimexpr\textwidth/2-\widthSmall\relax} 
    \hfill
    \subcaptionbox{\label{fig:speed_I}}[81pt]{%
                \includegraphics[clip, trim=9 0 427pt 20.4pt ]{Images/combined_speed_comparisons_fixed_adobe.pdf}}%
                \hspace{25pt}
    \subcaptionbox{\label{fig:speed_K}}[85pt]{%
                \includegraphics[clip, trim=123 0 308pt 20.4pt ]{Images/combined_speed_comparisons_fixed_adobe.pdf}}%
                \hspace{15.05pt}
    \subcaptionbox{\label{fig:speed_R}}[131pt]{%
                \includegraphics[clip, trim=230 0 156pt 20.4pt ]{Images/combined_speed_comparisons_fixed_adobe.pdf}}%
                \hspace{4pt}
    \subcaptionbox{\label{fig:speed_M}}[131pt]{%
                \includegraphics[clip, trim=379.9 0 6.8pt 20.4pt ]{Images/combined_speed_comparisons_fixed_adobe.pdf}}%
    \hfill
    \null

    \caption{Execution time comparisons of original methods for varying
    graph sizes and embedding sizes. Log-log plots are used for all
    experiments, with each plot having powers of ten in its vertical axis
    scaled equally to powers of ten in its horizontal axis. Lines represent
    medians, while shaded areas represent 25-th and 75-th percentiles.}
    \label{fig:speed}
\end{figure*}

Here we experimentally compare the time complexity of all original methods
for increasing graph sizes and embedding sizes. Note that since in this
experiment our interest is in evaluating the time complexity of the
embedding generation of each method, we will omit the normalization and the
clustering of the embeddings from our measurements. For the complete details
of the experiment setup, please refer to \autoref{sec:exec_time_setup}.

\subsubsection{Results Analysis.}

All results are presented in \autoref{fig:speed}. However, note that, due to
out-of-memory errors, only partial results are presented for CMNC.

In \autoref{fig:speed_I}, we see that all methods present an approximately
polynomial time complexity with regards to the number of nodes, and that
ComClus performs very similarly to Symmetric Richcom. Also, while \method{}
exhibits a growth rate similar to ComClus and Symmetric Richcom, it is about
3 to 4 times faster than both of them. CMNC performs the worst among all
methods while also exhibiting a higher growth rate.

Based on \autoref{fig:speed_K}, we reach similar conclusions with regards to
the number of views as well. The main difference here is that
\method{} exhibits even better performance, by being about 7 to 9
times faster than Symmetric Richcom, and about 18 to 21 times faster than
ComClus. Also, Symmetric Richcom is now about 2.5 to 3 times faster than
ComClus. 

In \autoref{fig:speed_R}, we notice that although all methods exhibit an
approximately polynomial time complexity with respect to the dimension of
node embeddings, they seem to have very different growth rates from each
other.  CMNC seems to again be having the worse performance, but due to
out-of-memory errors, it is hard to make strong claims here.  Also,
ComClus exhibits a much worse growth rate compared to Symmetric Richcom.
\method{} presents once again the best performance, and in fact it exhibits
a near-linear time complexity. That said, we suspect that this may be
misleading, as our implementation of \method{} often unnecessarily computes
full eigendecompositions, even though typically only a portion of the
eigenpairs is required.

Lastly, in \autoref{fig:speed_M}, we see that \method{} again significantly
outperforms the baselines, while CMNC produced out-of-memory errors for the
entirety of this experiment. An advantage of ComClus is that it exhibits a
linear time complexity in this case, while \method{} and Symmetric Richcom
seem to still have an approximately polynomial time complexity.

\section{Conclusions.}

In this work, we devised a unifying framework for multiple graph clustering
    methods, which is capable of modeling data as complex as multi-view
    graphs with multiple view structures. Then, we proposed \method{}, a
    novel instance of this framework, which also aims to have principled
    foundations by virtue of being a generalization of the highly successful
    spectral clustering.  Additionally, we conducted in-depth experiments on
    artificial data, in which we controlled for every aspect of the
    clustering workflow. These experiments demonstrated that \method{} can
    have similar or better clustering performance than the baselines, while
    also being more computationally efficient. Lastly, we evaluated
    \method{} on a real-world multi-view graph with multiple view
    structures, which showed that it can effectively model such complex
    datasets and uncover meaningful insights.

\section*{Acknowledgements}
This research was supported by the National Science Foundation under CAREER
    grant no. IIS 2046086, grant no. IIS 1901379, and CREST Center for
    Multidisciplinary Research Excellence in Cyber-Physical Infrastructure
    Systems (MECIS) grant no. 2112650.

\printbibliography
\end{refsection}
\newpage
\appendix
\begin{refsection}
\section{Appendix.}

\subsection{Related Work.}\label{sec:prob_form}

\subsubsection{ComClus.}
\label{sec:ComClus}
ComClus \cite{ni2017comclus} operates on tensors whose frontal slices are
symmetric adjacency matrices with non-negative elements. Specifically, it
models the $k$-th view by approximating $\mat{X}^{(k)}$ as
$\mat{O}^{(k)}\mat{U} \mat{D}_\mat{W}^{(k)} (\mat{O}^{(k)}\mat{U})^T$, where
$\mat{U}$ and $\mat{W}$ are factor matrices of size $I \times R$ and $K
\times R$, respectively, and $\mat{D}_\mat{W}^{(k)}$ is defined as
$\diag{\mat{W}_{k:}}$. Also, each $\mat{O}^{(k)}$ is a predefined indicator
matrix that accounts for the fact that a node that is shared between
different views may be represented by different rows
and columns in the corresponding adjacency matrices. ComClus additionally
defines factor matrices $\mat{A}$ and $\mat{B}$ of sizes $K \times M$ and $M
\times R$, respectively, which are then used to model $\mat{W}$ as
$\mat{AB}$, where $\mat{A}$ is constrained to be an indicator matrix, and
the rows of $\mat{B}$ are the latent representations of the view clusters.
Also, for computational tractability reasons the authors relax the
constraint of $\mat{A}$ and impose sparsity and non-negativity on $\mat{U}$,
$\mat{A}$ and $\mat{B}$, while $\mat{W}$ is constrained to be non-negative.
Then, the authors calculate this model via the following optimization
problem:
\begin{multline}
    \inf_{\mat{U},\mat{W},\mat{A},\mat{B}\geq 0} \sum_{k=1}^K
    \norm{\mat{X}^{(k)}- \mat{O}^{(k)}\mat{U} \mat{D}_\mat{W}^{(k)}
    (\mat{O}^{(k)}\mat{U})^T }^2+
    \\
    r(\mat{U},\mat{W},\mat{A},\mat{B})
    \label{eq:comclus}
\end{multline}
where
$r(\mat{U},\mat{W},\mat{A},\mat{B}):=\beta||\mat{W}-\mat{A}\mat{B}||^2+\rho(||\mat{U}||_1+||\mat{A}||_1+||\mat{B}||_1)$,
and $\beta$ and $\rho$ are penalty parameters that need to be defined by the
user. Each factor matrix is updated individually via multiplicative updates
in a block coordinate descent fashion \cite{wright2015coordinate}, until the
value of the objective function stops improving. Lastly, the authors assign
the $m$-th view to the $n$-th cluster when the $n$-th element of the $m$-th
row of $\mat{A}$ has the largest magnitude among all elements of that row.
Then, they consider the embeddings of the nodes of the $n$-th calculated
view cluster to be the rows of $\mat{U}\diag{\mat{B}_{n:}}$ and cluster the
nodes in the same way.

Note that the original formulation of ComClus includes two additional terms.
Specifically, one term forces the latent representations of two views to be
more orthogonal to each other as the number of their mutual nodes decreases,
while the other term enables the user to perform semi-supervised learning
when there is additional available information about how the various views
relate to each other. While in this work we omit these terms, note that the
effect of the term that imposes orthogonality can be achieved implicitly in
a different way. That is, instead of considering an edge between two
unshared nodes as unknown, we can consider it as known with weight zero. In
this way, the larger the number of unshared nodes between two views
 is, the larger the number of elements which are
non-zero in only one of the corresponding adjacency matrices of the networks
will be. This implies that these adjacency matrices will tend to be
orthogonal to each other, a property which will tend to hold for their
latent representations as well. In fact, the original ComClus formulation
can be interpreted as operating under the open world assumption
\cite{nickel2015review}, while our modification can be seen as operating
under the closed world assumption\reminder{needs citation too?}.

\reminder{Explain how improper scaling of columns of A and rows of B can affect final result}

\subsubsection{Methods Based on Block Term Decomposition.}
\paragraph{Richcom.} Richcom \cite{gujral2020beyond} applies a rank-$(L_m,
L_m, 1)$ terms decomposition \cite{de2008decompositions} on the data tensor
and operates on tensors whose frontal slices can be arbitrary adjacency
matrices with non-negative elements. Specifically, it considers a tensor
$\tensor{X}$ of size $I \times J \times K$ and models the graph by
approximating $\tensor{X}$  as $\sum_{m=1}^M \left(\mat{U}^{(m)}
{\mat{V}^{(m)}}^T\right)\times_3 \mat{a}^{(m)}$, where $\mat{U}^{(m)}$,
$\mat{V}^{(m)}$ and $\mat{a}^{(m)}$ are factors of sizes $I\times R_m$,
$J\times R_m$ and $K$, respectively. The authors implicitly assume that the
ordering of the nodes is identical for all views, and, therefore, no
permutations similar to these of ComClus are required. Also, note that all
adjacency matrices need to be of size $I \times J$, which can be
interepreted as either that Richcom only works when all nodes exist in all
views, or that it makes the closed world assumption. Note that all factor
matrices are constrained to be sparse and non-negative. Then, the authors
calculate their model via the following optimization problem:
\begin{multline}
    \inf_{\substack{\mat{U}^{(m)},\mat{V}^{(m)},\\ \mat{a}^{(m)} }}
    \norm{\tensor{X}-\sum_{m=1}^M \left(\mat{U}^{(m)}
    {\mat{V}^{(m)}}^T\right)\times_3 \mat{a}^{(m)}}^2
    +\\
    \sum_{m=1}^M r\left(\mat{U}^{(m)},\mat{V}^{(m)},\mat{a}^{(m)}\right)
    \label{eq:richcom}
\end{multline}
where $r\left(\mat{U}^{(m)},\mat{V}^{(m)},\mat{a}^{(m)}\right)$ encodes the
sparsity and non-negativity constraints\reminder{ask ekta}. They solve
\eqref{eq:richcom} using the AO-ADMM framework  \cite{huang2016flexible}
which solves for all $\mat{U}^{(m)}$, $\mat{V}^{(m)}$ and $\mat{a}^{(m)}$
individually in an alternating fashion via the alternating direction method
of multipliers \cite{boyd2011distributed}. Lastly, the authors assign the
$k$-th view to the $m$-th view cluster if $\mat{a}^{m}_k \geq
\mat{a}^{n}_k$ for all $n$, and then consider
$\mat{U}^{(m)}{\mat{V}^{(m)}}^T$ as an adjacency matrix representing the
$m$-th view structure and  obtain the corresponding nodes clustering by
extracting its weakly connected components after proper thresholding. 

\paragraph{Centroid-based Multilayer  Network Clustering  (CMNC).} CMNC
\cite{chen2019tensor} also applies a rank-$(L_m, L_m, 1)$ terms
decomposition and its model is very similar to that of Richcom with their
only differences being that for CMNC $\mat{U}^{(m)}=\mat{V}^{(m)}$ for all
$m$,  $\sum_{m=1}^M \mat{a}^{(m)}_i=1$ for all $i$, and no sparsity is
imposed. It is also different in the way the optimization problem was solved
which was reformulated into an unconstrained one by introducing two
differentiable operators, and then a trust region optimization method was
applied \cite{nocedal2006numerical}. Another difference is that CMNC
preprocesses each view of the data tensor as $
{\mat{D}^{(k)}}^{-\frac{1}{2}}\mat{X}^{(k)}{\mat{D}^{(k)}}^{-\frac{1}{2}}$
where $\mat{D}^{(k)}:=\diag{\sum_{i=1}^I\mat{X}^{(k)}_{:i}}$ inspired by the
spectral clustering paradigm \cite{von2007tutorial}. Lastly, another
difference is that the nodes clusters of the $m$-th view structure are
formed by assigning the $i$-th node to the cluster corresponding to the
largest element of the $i$-th row of $\mat{U}^{(m)}$.

\subsubsection{Spectral Clustering.} 
\label{sec:single_view_spec_clus} 
Spectral clustering
\cite{ng2001spectral}, \cite{von2007tutorial}, \cite{zhou2005learning}, \cite{klus2022koopman}
has seen great success and developments in the past decades thanks to its
strong theoretical foundations and its ability to discover clusters of
arbitrary shape. In fact, these developments have led to generalized
versions for multi-view graphs
\cite{liu2012multiview}, \cite{zhou2007spectral}. 

Consider an adjacency matrix $\mat{X}$ of an arbitrary undirected graph with
non-negative weights. If we define
$\mat{D}:=\diag{\sum_{i=1}^I\mat{X}_{:i}}$ then $\mat{L}:=\mat{D}-\mat{X}$
is called the Laplacian of $\mat{X}$. It can be shown \cite{von2007tutorial}
that $\mat{L}$ is positive semi-definite and that the number of connected
components of the graph is equal to the multiplicity of the smallest
eigenvalue of $\mat{L}$, which is always 0. It can also be shown
\cite{ng2001spectral} that if $\mat{U}$ is a matrix whose columns form a
basis for the eigenspace corresponding to the smallest eigenvalue of
$\mat{L}$, then two rows of $\mat{U}$ are collinear if the corresponding
nodes belong to the same connected component, and orthogonal to each other
if the corresponding nodes belong to different components.  Spectral
clustering algorithms then use these properties to cluster the rows of
$\mat{U}$ and identify the communities of a graph in a principled manner.
Additionally,
$\mat{L}_{sym}:={\mat{D}}^{-\frac{1}{2}}\mat{L}{\mat{D}}^{-\frac{1}{2}}$ is
called the normalized Laplacian and inherits all the aforementioned nice
properties of $\mat{L}$. The usefulness of $\mat{L}_{sym}$ is usually
justified in the literature by associating it to a relaxation of the $n$-cut
problem \cite{von2007tutorial}. Lastly, note that the eigenspace of
$\mat{L}_{sym}$ corresponding to an eigenvalue of 0 is identical to the
eigenspace of $\mat{S} := \mat{I}-\mat{L}_{sym} =
{\mat{D}}^{-\frac{1}{2}}\mat{X}{\mat{D}}^{-\frac{1}{2}}$ corresponding to
its maximum eigenvalue which is 1.

Note there is a more direct and intuitive justification for choosing
$\mat{L}_{sym}$ over $\mat{L}$. Specifically, notice that since
$\mat{L}_{sym}$ is positive semi-definite, the eigenvalues of $\mat{S} :=
\mat{I}-\mat{L}_{sym} =
{\mat{D}}^{-\frac{1}{2}}\mat{X}{\mat{D}}^{-\frac{1}{2}}$  will all be less
than or equal to $1$, and since all elements of $\mat{S}$ are non-negative,
it can in turn be shown that its  smallest eigenvalue will also be greater
than or equal to $-1$ \reminder{cite}. Therefore, the eigenvalues of
$\mat{L}_{sym}$ will be bounded between 0 and 2. This property can be
especially useful when the different communities of the graph are not
completely disconnected from each other, in which case we evaluate the
number of communities to be equal to the number of eigenvalues of
$\mat{L}_{sym}$ that are only approximately 0, or the number of eigenvalues
of $\mat{S}$ that are only approximately 1. Another reason is that when
communities are completely disconnected from each other, then
$\mat{L}_{sym}$ can be expressed as a block-diagonal matrix where each block
corresponds to a different community. This implies that the set of the
eigenvalues of $\mat{L}_{sym}$ is the union of the eigenvalues of its
blocks,\reminder{cite} and, therefore, the maximum eigenvalue for all
communities will be 2. In turn, this implies that when communities are not
entirely disconnected from each other, the decision of whether an eigenvalue
is close enough to 0 does not have to involve the size of the corresponding
community.

\subsubsection{Multi-View Spectral Clustering.}
When a graph has multiple views such that all views are assumed to share a
common underlying structure, i.e. $M = 1$, then one of the multi-view
spectral clustering models proposed in
\cite{liu2012multiview}, \cite{zhou2007spectral} can be applied. In our work,
we will be particularly interested in some of the models proposed in
\cite{liu2012multiview}. Specifically, we will study the MC-TR-I model, 
\begin{equation}
    \label{eq:multi_unweighted_appendix}
    \sup_{\mat{U}^T\mat{U}=\mat{I}} \sum_{k=1}^K\tr{\mat{U}^T\mat{S}^{(k)}\mat{U}},
\end{equation}
which aims to perform spectral clustering jointly on all views in a way that
the embeddings for all views are identical to each other. We will also
consider its weighted variant, 
\begin{equation}
    \label{eq:multi_weighted_appendix}
    \sup_{\substack{\mat{U}^T\mat{U}=\mat{I}\\\mat{a}\succeq\mat{0}, \norm{\mat{a}}=1
    }} \sum_{k=1}^K\tr{\mat{U}^T\mat{a}_k\mat{S}^{(k)}\mat{U}},
\end{equation}
which assigns a different weight to each view. To calculate these
models, the authors proposed the MC-TR-I-EVD and MC-TR-I-EVDIT algorithms,
respectively.

Then, in \cite{kumar2011co} the authors propose two forms of  co-regularized
spectral clustering with the first form consisting of a pairwise
co-regularization as
\begin{multline}
    \label{eq:pairwise_coreg_appendix}
    \sup_{{\mat{U}^{(k)}}^T\mat{U}^{(k)}=\mat{I}}
    \sum_{k=1}^K\tr{{\mat{U}^{(k)}}^T\mat{S}^{(k)}\mat{U}^{(k)}}+
    \\
    \lambda \sum_{i\neq j}\tr{\mat{U}^{(i)}{\mat{U}^{(i)}}^T\mat{U}^{(j)}{\mat{U}^{(j)}}^T}
\end{multline}
and the second one consisting of a centroid-based co-regularization as
\begin{multline}
    \label{eq:centroid_coreg_appendix}
    \sup_{\substack{{\mat{U}^{(k)}}^T\mat{U}^{(k)}=\mat{I}\\
    {\mat{U}^*}^T\mat{U}^*=\mat{I}}}
    \sum_{k=1}^K\tr{{\mat{U}^{(k)}}^T\mat{S}^{(k)}\mat{U}^{(k)}}
    +
    \\
    \lambda_k \tr{\mat{U}^{(k)}{\mat{U}^{(k)}}^T\mat{U}^*{\mat{U}^*}^T}.
\end{multline}
Note that both of these formulations can be seen as relaxed versions of
\eqref{eq:multi_unweighted_appendix} where instead the  embeddings from different
views are forced to only be similar to each other instead of exactly equal.
Specifically, the regularization terms force the columnspaces of all
$\mat{U}^{(k)}$ to  become more similar with each other as $\lambda$ and
$\lambda_k$ take larger values. In the limit, their columnspaces become
identical and the  rows  of $\mat{U}^{(k)}$ will be just orthogonally
rotated versions of the rows of any other  $\mat{U}^{(i)}$. Therefore,
applying k-means, which is the clustering method suggested by the authors,
on the rows of any $\mat{U}^{(k)}$ will lead to the same node clusters.

Now notice that we can show that \eqref{eq:multi_weighted_appendix} has the same
optimal set as
\begin{equation}
 \arginf_{
     \substack{\mat{U}^T\mat{U}=\mat{I}\\\mat{A}\succeq\mat{0}, \norm{\mat{A}}=1}
 }
    \norm{\tensor{S}-\cp{U}{U}{AB}{}{}{}}
    \label{eq:multi_view_spectral_appendix}
\end{equation}
where $\tensor{S}_{::k}:=\mat{S}^{(k)}$, $\mat{A}:=\mat{a}$ and
$\mat{B}:=\mat{1}^T$
(see \autoref{sec:multi_view_spectral_proof} for proof).
Similarly, we can show that \eqref{eq:multi_unweighted_appendix} has the
same optimal set as $\inf_{\mat{U}^T\mat{U}=\mat{I}
}\norm{\tensor{S}-\cp{U}{U}{AB}{}{}{}}$ where
$\tensor{S}_{::k}:=\mat{S}^{(k)}$, $\mat{A}:=\mat{1}$ and
$\mat{B}:=\mat{1}^T$. Lastly, note that due to the close connection of
\eqref{eq:pairwise_coreg_appendix} and \eqref{eq:centroid_coreg_appendix}
with \eqref{eq:multi_unweighted_appendix} one can argue that the same
reformulation captures the essence of these problems as well, despite the
fact that they cannot be reformulated exactly like that from a strictly
mathematical point of view.

   \begin{table*}
       \centering
       \begin{tabular}{ |c|c|  }
           \hline
           \textbf{Model Name} & \textbf{Model Definition}\\
           \hline
           Spectral Clustering \cite{ng2001spectral} & \parbox{10cm}{
               \begin{align*}
                   \sup_{\mat{U}^T\mat{U}=\mat{I}} \tr{\mat{U}^T\left(\mat{D}^{-\frac{1}{2}}\mat{X}\mat{D}^{-\frac{1}{2}}\right)\mat{U}}
               \end{align*}
           }
           \\
           \hline
           MC-TR-I-EVD \cite{liu2012multiview} & \parbox{10cm}{
               \begin{align*}
                   \sup_{\mat{U}^T\mat{U}=\mat{I}} \sum_{k=1}^K\tr{\mat{U}^T\left({\mat{D}^{(k)}}^{-\frac{1}{2}}\tensor{X}_{::k}{\mat{D}^{(k)}}^{-\frac{1}{2}}\right)\mat{U}}
               \end{align*}
               }
               \\
               \hline
               MC-TR-I-EVDIT \cite{liu2012multiview} & \parbox{10cm}{\begin{align*}
                   \sup_{\substack{\mat{U}^T\mat{U}=\mat{I}\\\mat{a}\succeq\mat{0}, \norm{\mat{a}}=1
                   }} \sum_{k=1}^K\tr{\mat{U}^T\left(\mat{a}_k{\mat{D}^{(k)}}^{-\frac{1}{2}}\tensor{X}_{::k}{\mat{D}^{(k)}}^{-\frac{1}{2}}\right)\mat{U}}
               \end{align*}
               }
               \\
               \hline
               \parbox{3cm}{\centering Co-regularized\\Spectral Clustering\\
               (pairwise regularization)} \cite{kumar2011co} & \parbox{10cm}{
                   \begin{multline*}
                       \sup_{{\mat{U}^{(k)}}^T\mat{U}^{(k)}=\mat{I}}
                       \sum_{k=1}^K\tr{{\mat{U}^{(k)}}^T\left({\mat{D}^{(k)}}^{-\frac{1}{2}}\tensor{X}_{::k}{\mat{D}^{(k)}}^{-\frac{1}{2}}\right)\mat{U}^{(k)}}
                       \\
                       +\lambda \sum_{i\neq j}\tr{\mat{U}^{(i)}{\mat{U}^{(i)}}^T\mat{U}^{(j)}{\mat{U}^{(j)}}^T}
                   \end{multline*}
                   }
               \\
               \hline
               \parbox{3cm}{\centering Co-regularized\\Spectral
               Clustering\\(centroid-based regularlization)} \cite{kumar2011co} & \parbox{10cm}{
                   \begin{multline*}
                       \sup_{\substack{{\mat{U}^{(k)}}^T\mat{U}^{(k)}=\mat{I}
                       \\
                       {\mat{U}^*}^T\mat{U}^*=\mat{I}}}
                       \sum_{k=1}^K\tr{{\mat{U}^{(k)}}^T\left({\mat{D}^{(k)}}^{-\frac{1}{2}}\tensor{X}_{::k}{\mat{D}^{(k)}}^{-\frac{1}{2}}\right)\mat{U}^{(k)}}
                       \\
                       +\lambda_k\tr{\mat{U}^{(k)}{\mat{U}^{(k)}}^T\mat{U}^*{\mat{U}^*}^T}
                   \end{multline*}
                   }
                   \\
                   \hline
                   ComClus \cite{ni2017comclus} & \parbox{10cm}{ 
                   \begin{multline*}
                       \inf_{\mat{U},\mat{W},\mat{A},\mat{B}\geq 0} \sum_{k=1}^K \norm{\mat{X}_{::k}- \mat{U} \diag{\mat{W}_{k:} }\mat{U}^T }^2\\+\beta||\mat{W}-\mat{A}\mat{B}||^2+\rho(||\mat{U}||_1+||\mat{A}||_1+||\mat{B}||_1)
                   \end{multline*}}\\
                   \hline
                   CMNC \cite{chen2019tensor} & \parbox{10cm}{ \begin{gather*}
                       \inf_{\substack{\mat{U}^{(m)}, \mat{A}\geq 0 \\ \mat{A1}=\mat{1}} } \norm{\tensor{Y}-\sum_{m=1}^M \left(\mat{U}^{(m)} {\mat{U}^{(m)}}^T\right)\times_3 \mat{A}_{:m}}^2
                       \\
                       \text{where}
                       \\
                       \tensor{Y}_{::k}:={\mat{D}^{(k)}}^{-\frac{1}{2}}\tensor{X}_{::k}{\mat{D}^{(k)}}^{-\frac{1}{2}}
                   \end{gather*}
                   }
                   \\
                   \hline
                   Richcom \cite{gujral2020beyond} & \parbox{10cm}{
                       \begin{multline*}
                           \inf_{\substack{\mat{U}^{(m)},\mat{V}^{(m)},\\ \mat{a}^{(m)} }} \norm{\tensor{X}-\sum_{m=1}^M \left(\mat{U}^{(m)} {\mat{V}^{(m)}}^T\right)\times_3 \mat{a}^{(m)}}^2
                           \\+
                           \sum_{m=1}^M r\left(\mat{U}^{(m)},\mat{V}^{(m)},\mat{a}^{(m)}\right)
                       \end{multline*}
                       where $r\left(\mat{U}^{(m)},\mat{V}^{(m)},\mat{a}^{(m)}\right)$ encodes sparsity and non-negativity\\}
                       \\
                       \hline
       \end{tabular}
       \caption{Mathematical formulations of various graph clustering methods.}
       \label{tb:methods}
   \end{table*}

\subsection{Proofs}

\subsubsection{Derivation of \eqref{eq:multi_view_spectral}.}
\label{sec:multi_view_spectral_proof}
{
\allowdisplaybreaks
\begin{align*}
        &\argsup_{\substack{\mat{U}^T\mat{U}=\mat{I}\\\mat{a}\succeq\mat{0}, \norm{\mat{a}}=1
    }} \tr{\mat{U}^T\sum_{k=1}^K\mat{a}_k\mat{S}^{(k)}\mat{U}} =
    \\
    &\arginf_{\substack{\mat{U}^T\mat{U}=\mat{I}\\\mat{a}\succeq\mat{0}, \norm{\mat{a}}=1
    }} - 2\tr{\mat{U}^T\sum_{k=1}^K\mat{a}_k\mat{S}^{(k)}\mat{U}}+R = 
    \\
    &\arginf_{\substack{\mat{U}^T\mat{U}=\mat{I}\\\mat{a}\succeq\mat{0}, \norm{\mat{a}}=1
    }}\sum_{k=1}^K -2\tr{\mat{U}^T\mat{a}_k\mat{S}^{(k)}\mat{U}}+ {\mat{a}_k}^2 R =
    \\
    &\arginf_{\substack{\mat{U}^T\mat{U}=\mat{I}\\\mat{a}\succeq\mat{0}, \norm{\mat{a}}=1
    }}\sum_{k=1}^K \tr{-2\mat{a}_k\mat{U}\mat{U}^T\mat{S}^{(k)}+ \left(\mat{a}_k \mat{U}\mat{U}^T\right)^2} =
    \\
    &\arginf_{\substack{\mat{U}^T\mat{U}=\mat{I}\\\mat{a}\succeq\mat{0}, \norm{\mat{a}}=1
    }}\sum_{k=1}^K\norm{\mat{S}^{(k)}-\mat{a}_k \mat{U}\mat{U}^T}^2 =
    \\
    &\arginf_{\substack{\mat{U}^T\mat{U}=\mat{I}\\\mat{A}\succeq\mat{0}, \norm{\mat{A}}=1
    }}\norm{\tensor{S}-\cp{U}{U}{AB}{}{}{}}
\end{align*}
}
where $\tensor{S}_{::k}:=\mat{S}^{(k)}$, $\mat{A}:=\mat{a}$ and $\mat{B}:=\mat{1}^T$.

\subsubsection{Best Positive Semi-Definite Approximation.}
\label{sec:psd_approx_proof}
\begin{theorem}
    \label{th:psd_approx}
    If $\mat{Y} \in \mathbb{R}^{I \times I}$ is a symmetric matrix with $E$
    non-negative eigenvalues, then a
    positive semi-definite matrix, $\mat{S} \in \mathbb{R}^{I \times I}$,
    that minimizes $\norm{\mat{Y}-\mat{S}}$ such that $\rank{\mathbf{S}}\leq
    R$ has an eigenvalue decomposition consisting of the largest
    $\max\{R,E\}$ non-negative eigenvalues of $\mat{Y}$ and their
    corresponding eigenvectors.
\end{theorem}
\begin{proof}
    If we denote the eigenvalue decompositions of $\mat{Y}$ and $\mat{S}$ as
    $\mat{U}\mat{\Lambda}\mat{U}^T$ and $\mat{V}\mat{\Sigma}\mat{V}^T$,
    respectively, such that the eigenvalues are sorted in a descending order,
    then we have 
    \begin{equation*}
        \arginf_{\substack{
            \mat{S}\succeq\mat{0}\\
            \rank{\mat{S}}\leq R
        }} \norm{\mat{Y}-\mat{S}}=
        \arginf_{\substack{
            \mat{S}\succeq\mat{0}\\
            \rank{\mat{S}}\leq R
        }} \frac{1}{2}\tr{\mat{S}^2}-\tr{\mat{Y}^T \mat{S}}
    \end{equation*}
    or equivalently
    \begin{align*}
        \arginf_{\substack{
            \mat{V}^T\mat{V}=\mat{I}\\
            \rank{\mat{\Sigma}}\leq R\\
            \mat{\Sigma}_{ii} \geq 0\\
            \forall i\neq j, \mat{\Sigma}_{ij}=0}} \frac{1}{2}\sum_{i=1}^I\mat{\Sigma}_{ii}^2-\tr{\mat{U}\mat{\Lambda}\mat{U}^T\mat{V}\mat{\Sigma}\mat{V}^T}.
    \end{align*}
    From \cite{theobald1975inequality} we know that $\tr{\mat{U}\mat{\Lambda}\mat{U}^T\mat{V}\mat{\Sigma}\mat{V}^T} \leq \sum_{i=1}^I\mat{\Lambda}_{ii}\mat{\Sigma}_{ii}$, and notice that the equality can be achieved if we set $\mat{V}=\mat{U}$. Then, we can find an optimal $\mat{\Sigma}$ by solving
    \begin{align*}
        \arginf_{\substack{
            \rank{\mat{\Sigma}}\leq R\\
            \mat{\Sigma}_{ii} \geq 0\\
            \forall i\neq j, \mat{\Sigma}_{ij}=0}}
        \sum_{i=1}^I\frac{1}{2}\mat{\Sigma}_{ii}^2-\mat{\Lambda}_{ii}\mat{\Sigma}_{ii}.
    \end{align*}
    To this end, notice first that for any $i$ we have
    \begin{equation*}
        \inf_{
            \mat{\Sigma}_{ii} \geq 0}
        \frac{1}{2}\mat{\Sigma}_{ii}^2-\mat{\Lambda}_{ii}\mat{\Sigma}_{ii} = 
        \left\{
            \begin{array}{ll}
                -\frac{1}{2}\mat{\Lambda}_{ii}^2 & \text{if } \mat{\Lambda}_{ii}\geq 0\\
                0 & \text{if } \mat{\Lambda}_{ii}< 0
            \end{array}
            \right.
    \end{equation*}
    and, therefore, the optimal $\mat{\Sigma}$ is given by setting
    \begin{equation*}
        \mat{\Sigma}_{ii} = 
        \left\{
            \begin{array}{ll}
                \max \left\{\mat{\Lambda}_{ii},0 \right\} & \text{if } i\leq R\\
                0 & \text{if } i > R
            \end{array}
            \right.
    \end{equation*}
\end{proof}

\subsubsection{Derivation of \eqref{eq:proposed1_2}.}
\label{sec:proposed1_2_proof}

\begin{flalign}
\nonumber
     \arginf_{\substack{{\mat{U}^{(m)}}^T\mat{U}^{(m)}=\mat{I} \\ \mat{B}^T\in \mathcal{I}}} 
     & 
     \sum_{m=1}^M \sum_{k=1}^K\norm{\tensor{Y}_{::k}-\mat{A}_{km}
     \mat{U}^{(m)}{\mat{U}^{(m)}}^T}^2=
     &
    \\\nonumber
    \arginf_{\substack{{\mat{U}^{(m)}}^T\mat{U}^{(m)}=\mat{I}\\ \mat{B}^T\in \mathcal{I}}}
    &
        \sum_{m=1}^M \sum_{k=1}^K
 \tr{\left(\mat{A}_{km}\mat{U}^{(m)}{\mat{U}^{(m)}}^T\right)^2}+
    &
    \\
       \noalign{\hfill $\displaystyle -2\tr{\mat{A}_{km}\mat{U}^{(m)}{\mat{U}^{(m)}}^T\tensor{Y}_{::k}}=$}\nonumber
     \arginf_{\substack{{\mat{U}^{(m)}}^T\mat{U}^{(m)}=\mat{I}\\ \mat{B}^T\in \mathcal{I}}}
    &
     \sum_{m=1}^M\sum_{k=1}^K 
     \tr{{\mat{A}_{km}}^2{\mat{U}^{(m)}}^T\mat{U}^{(m)}}+
     &
     \\
     \noalign{\hfill $\displaystyle -2\tr{{\mat{U}^{(m)}}^T \mat{A}_{km} \tensor{Y}_{::k}\mat{U}^{(m)}}=$}\nonumber
     \argsup_{\substack{{\mat{U}^{(m)}}^T\mat{U}^{(m)}=\mat{I}\\ \mat{B}^T\in \mathcal{I}}}
     &
     \sum_{m=1}^M \tr{{\mat{U}^{(m)}}^T \mat{Z}^{(m)}\mat{U}^{(m)}}
     &
\end{flalign}
where  $\mat{Z}^{(m)}:=\sum_{k=1}^K 2\mat{A}_{km}\tensor{Y}_{::k}-\norm{\mat{A}_{:m}}^2\mat{I}$.

\subsubsection{Derivation of \eqref{eq:proposed2}.}
\label{sec:proposed2_proof}
\begin{flalign*}
    \arginf_{\substack{{\mat{U}^{(m)}}^T\mat{U}^{(m)}=\mat{I} \\ \mat{B}^T\in \mathcal{I}}}
    &
    \sum_{m=1}^M \sum_{k=1}^K \norm{\tensor{Y}_{::k}-\mat{A}_{km}\mat{U}^{(m)}\mat{D}_{\mat{B}}^{(m)}{\mat{U}^{(m)}}^T}^2=
    &
    \\
    \arginf_{\substack{{\mat{U}^{(m)}}^T\mat{U}^{(m)}=\mat{I}\\ \mat{B}^T\in \mathcal{I}}}
    &
    \sum_{m=1}^M \sum_{k=1}^K \tr{{\mat{A}_{km}}^2{\mat{D}_{\mat{B}}^{(m)}}^2}+ 
    &
    \\ 
    \noalign{\hfill $\displaystyle \tr{-2\mat{A}_{km}\mat{U}^{(m)}\mat{D}_{\mat{B}}^{(m)}{\mat{U}^{(m)}}^T\tensor{Y}_{::k}}=$}
    \arginf_{\substack{{\mat{U}^{(m)}}^T\mat{U}^{(m)}=\mat{I}\\ \mat{B}^T\in \mathcal{I}}}
    &
    \sum_{m=1}^M \tr{\norm{\mat{A}_{:m}}^2{\mat{D}_{\mat{B}}^{(m)}}^2}+
    &
    \\
    \noalign{\hfill $\displaystyle \tr{-2\mat{U}^{(m)}\mat{D}_{\mat{B}}^{(m)}{\mat{U}^{(m)}}^T \sum_{k=1}^K \mat{A}_{km}\tensor{Y}_{::k}}=$}
    \arginf_{\substack{{\mat{U}^{(m)}}^T\mat{U}^{(m)}=\mat{I}\\ \mat{B}^T\in \mathcal{I}}}
    &
    \sum_{m=1}^M \tr{\norm{\mat{A}_{:m}}^2{\mat{D}_{\mat{B}}^{(m)}}^2}+
    &
    \\
    \noalign{\hfill $\displaystyle \tr{-2\norm{\mat{A}_{:m}}\mat{U}^{(m)}\mat{D}_{\mat{B}}^{(m)}{\mat{U}^{(m)}}^T \mat{Z}^{(m)}}=$}
    \arginf_{\substack{{\mat{U}^{(m)}}^T\mat{U}^{(m)}=\mat{I} \\ \mat{B}^T\in \mathcal{I}}}
    &
    \sum_{m=1}^M \norm{\mat{Z}^{(m)}-\norm{\mat{A}_{:m}}\mat{U}^{(m)}\mat{D}_{\mat{B}}^{(m)}{\mat{U}^{(m)}}^T}^2
    &
\end{flalign*}
where $\mat{Z}^{(m)}:=\sum_{k=1}^K\mat{A}_{km}\tensor{Y}_{::k}/\norm{\mat{A}_{:m}}$.

\subsection{Derivation of Space \& Time Complexity} 
\label{app:space_time_complexity}

Note that $R\leq MI$ and that the following discussion applies for any
combination of constraints on $\mat{A}$ an $\mat{B}$.

\paragraph{Space complexity.}

First notice that memory usage is maximized during the computation of
$\mat{A}$, which requires $\mat{U}$, $\mat{B}$ and $\cp{U}{U}{B}{}{}{}$ to
be available. $\mat{A}$ requires space $O\left(K\right)$ since each of its
rows has only one non-zero element, and similarly $\mat{B}$ requires space
$O\left(R\right)$. Also, $\mat{U}$ and $\cp{U}{U}{B}{}{}{}$ require space
$O\left(RI\right)$ and $O\left(MI^2\right)$, respectively. Then, during the
calculation of $\mat{B}$, we also need space $O\left(MI^2\right)$ to store
all $\mat{Z}^{(m)}$ along with their eigendecompositions which are
calculated sequentially.  However, this does not increase the total space
complexity since the existing $\cp{U}{U}{B}{}{}{}$ that was stored during
the previous calculation of $\mat{A}$ can be discarded and overwritten.
However, the eigenvalues of all $\mat{Z}^{(m)}$ will still need to all be
simultaneously stored, requiring an additional space of $O\left(MI\right)$.
Therefore, the total space complexity of our method is
$O\left(MI^2+K\right)$.

\paragraph{Time complexity.}

First we see that forming $\mat{Z}^{(m)}$ requires time $O\left(
\nnz{\mat{A}_{:m}} I^2\right)$, and, therefore, forming all of them requires
time $O\left( KI^2\right)$. Then, to  update $\mat{U}$ and $\mat{B}$, we
need the eigendecompositions of $\mat{Z}^{(m)}$ along with their sorted
combined eigenvalues, which require a time complexity of
$O\left(MI^3\right)$ and $O\left(MI\log\left(MI\right)\right)$,
respectively. Then, before updating $\mat{A}$ we need to compute
$\cp{U}{U}{B}{}{}{}$, and note that
$\cp{U}{U}{B}{}{}{}_{::m}=\mat{U}\diag{\mat{B}_{m:}}\mat{U}^T$ which costs
$O\left(\nnz{\mat{B}_{m:}}I^2\right)$. Therefore, the total time complexity
for computing $\cp{U}{U}{B}{}{}{}$ will be $O\left(RI^2\right)$. Next, we
are able to update $\mat{A}$ which costs $O\left(KMI^2\right)$. Thus, for
$t$ iterations the total time complexity of our method is
$O\left(tMI^3+tKMI^2+tMI\log\left(MI\right)\right)$.

\paragraph{Optimizations.}

Given the sparsity of a wide range of real-world graphs, the space
complexity of each $\mat{Z}^{(m)}$ can be greatly reduced to
$O\left(\nnz{\tensor{X}}\right)$ if a coordinate list representation is used
instead. In this case, the time complexity of each $\mat{Z}^{(m)}$ can also
be reduced to $O\left( \nnz{\mat{A}_{:m}} \nnz{\tensor{X}}\right)$, and,
therefore, the total time complexity of forming all of them will be $O\left(
K \nnz{\tensor{X}}\right)$. In fact, all $\mat{Z}^{(m)}$ can be calculated
in parallel, which although does not improve the worst-case time complexity,
may be very beneficial in practice.

Then, notice that, out of all eigenpairs of all $\mat{Z}^{(m)}$, we only
need the ones corresponding to the $R$ largest combined eigenvalues in order
to calculate $\mat{U}$ and $\mat{B}$. Therefore, we will only need the
eigenpairs corresponding to at most the $R$ largest eigenvalues of each
$\mat{Z}^{(m)}$ as well. Thus, when $R<I$, the computation cost of these
eigenpairs can be greatly reduced if computed by an appropriate eigensolver
that is also able to take advantage of sparsity
\cite{grimes1994shifted}\reminder{give exact complexity or references}. In
this case, the space complexity of each eigendecomposition will be reduced
to $O\left(IR\right)$ as well.  Additionally, calculating the eigenpairs of
different $\mat{Z}^{(m)}$ in parallel can further dramatically reduce their
total computation cost. Lastly, notice that using a min-heap can further
reduce the cost of finding the $R$ largest eigenvalues to $O\left(MR\log
R\right)$.

\subsection{Experiment Setup Details.}
\label{sec:experiment_details}
\reminder{split into two paragraphs}
Despite the fact that the embedding calculation method is the main novel
algorithmic contribution in the original papers where the baseline methods were
proposed, their authors have made algorithmic choices for the other segments as
well. However, it is not clear whether and why these choices would lead to
optimal clustering performance. Therefore, since we are mostly interested in
comparing the quality of the embeddings produced by each method, we define
enhanced versions for all methods by considering combinations of their
embeddings calculation method with various methods of data preprocessing,
embeddings postprocessing and embeddings clustering. Then, the combination that
produces the best clustering performance is reported as the performance of each
enhanced method. Specifically, for data preprocessing we consider both raw data
and the normalized Laplacians of the graph views. We use the normalized
Laplacians for directed graphs as defined in \cite{zhou2005learning} which is a
direct extension of the normalized Laplacians for undirected graphs we
discussed in \autoref{sec:spec_clus} and in more detail in
\autoref{sec:single_view_spec_clus}. Then, for the node embeddings
of the $n$-th calculated view cluster, we consider both their raw unnormalized
form and their normalization as either $\mat{U}\diag{\mat{B}_{n:}}$ or
$\mat{U}\diag{\mat{C}_{n:}}$ with $\mat{C}_{ij}:=\sqrt{|\mat{B}_{ij}|}$. After
this, embeddings may optionally be further normalized to be unit vectors.
Finally, we consider k-means, maximum likelihood and inner product thresholding
for clustering the embeddings.

Based on the theoretical arguments we made in \autoref{sec:MSSC}, we define
original \method{} to use normalized Laplacians for the preprocessing step, and
a non-negative constraint for $\mat{A}$ and $\mat{B}$. However, we still allow
it to be paired with any embedding post-processing and clustering method, in
the same way that enhanced \method{} does. Note that, although this design
choice may seem that it could prove an unfair advantage of \method{} over the
baselines, we believe that it still was the right decision for two main
reasons. First, as with the other baselines, we have no robust theoretical
justification for choosing a specific combination of embedding postprocessing
and clustering method. Second, we noticed in our experiments that the
performance of our method was not significantly affected by altering this
combination.

Lastly, we would also like to point out that due to potential issues that we
think we may have discovered with the optimization steps proposed in the
original paper of Richcom, we decided to consider a slightly different
variant in our experiments.  Specifically, we propose Symmetric Richcom,
which uses the same methods of preprocessing, postprocessing and embeddings
clustering as the original Richcom, but has $\mat{V}$ constrained to be
identical to $\mat{U}$. In \autoref{sec:richcom_solver}, we further explain
how to implement a solver for this Richcom variant.

\subsubsection{Software \& Hardware Specifications.}
\label{sec:soft_and_hard}
All experiments related to clustering performance were run on a machine with
two Intel\textsuperscript{\textregistered}
Xeon\textsuperscript{\textregistered} E5-2680 v3 processors, 378GB of RAM
and MATLAB 9.7.0.1190202 (R2019b). All timing experiments were run on MATLAB
9.7.0.1737446 (R2019b) Update 9 on a machine with an AMD
Ryzen\textsuperscript{\texttrademark} 5 5600 processor overclocked at
4.75GHz paired with 16GB of DDR4 dual-channel SDRAM and a 1600MHz memory
clock. Lastly, in all experiments the mtimex routine \cite{mtimesx} was used
for faster matrix multiplications.

\subsubsection{A Solver for Symmetric Richcom.}
\label{sec:richcom_solver}

We propose updating $\mat{U}$ and $\mat{A}$ according to the solver
described in \cite{seung2001algorithms}. In fact, this is the solver used
by ComClus for its updates of  $\mat{U}$, $\mat{A}$ and $\mat{B}$. This
provides further motivation for deriving the updates of Symmetric Richcom
in this manner in the sense that it may help reduce the impact the solvers
can have on the differences in the performance of the two models. To avoid
a lengthy derivation of the updates, notice that we can very easily deduce
them by directly modifying the solver of ComClus. Specifically, when
solving for $\mat{U}$ in symmetric Richcom we use the same update as in
ComClus but we substitute $\mat{W}$ with $\mat{AB}$. Similarly, when
solving for $\mat{A}$ we use the same update as in ComClus but we
substitute $\mat{W}$ with $\mat{Y}_{(3)}$ and $\mat{B}$ with
$\mat{B}(\mat{U}\odot\mat{U})^T$. Also, note that such a solver is also
capable of imposing sparsity and involves the same parameter $\rho$ as
ComClus.

\subsubsection{Clustering Quality on Artificial Data.}
\label{sec:artif_clust_perf_appendix}

\paragraph{Embeddings Calculation.}

Note that the following parametrization applies to both the original and the
enhanced methods, unless stated otherwise. First, we calculate 100 samples
of $\mat{U}$, $\mat{A}$ and $\mat{B}$ by modeling the data tensor as
$\cp{U}{U}{AB}{}{}{}$, where for each sample we consider a different random
initialization for $\mat{U}$, $\mat{A}$ and $\mat{B}$ and a different
instance of the same data tensor.  $R$ can range from 6 to 10, $M$ is set to
3, the convergence criterion is always defined to be the relative change in
the value of the objective function of each model, the convergence threshold
for the outer loop can be either $1e-3$, $1e-6$ or $1e-9$, and the maximum
number of iterations for the outer loop is set to 1000. Note that apart from
ComClus the other models consist of only one loop. Regarding enhanced
\method{}, we allow all nine possible combinations of constraints on
$\mat{A}$ and $\mat{B}$ discussed in \autoref{sec:optim_steps}, which is in
contrast with original \method{} which we defined to have only non-negative
constraints. Also, note we do not make use of any of the optimizations
discussed in \autoref{sec:space_time_complexity} and in more detail in
\autoref{app:space_time_complexity}.  For ComClus, $\rho$ can use values in
the set of the 6 equally spaced values from 0 to 0.16, $\beta$ can use
values in the set of the 6 equally spaced values from 0.01 to 0.9145, while
the threshold of convergence criterion of its inner loop is set to $1e-6$.
For CMNC we set its delta to 1, while for Symmetric Richcom $\rho$ can take
values in the set of the 6 equally spaced values from 0 to 0.2. Also note
that both CMNC and Symmetric Richcom require a fixed user-specified
$\mat{B}$, which in our experiments we always set it to reflect the ground
truth structure as closely as possible. Lastly, for both CMNC and Symmetric
Richcom, and when both $M$ and $R$ are larger than or equal to their ground
truth values, we begin by creating the $\mat{B}$ of optimal dimensions and
structure. Then, any remaining rows are added as all-zeros, and in turn any
remaining columns are added as random indicator vectors.  Also, when $M$ or
$R$ is less than the ground truth, then an appropriate number of rows or
columns, respectively, is removed randomly.

\paragraph{Matching calculated and ground truth view clusters.}

After the model is calculated, we assign the $m$-th view to the $n$-th
cluster when the $n$-th element of the $m$-th row of $\mat{A}$ has the
largest magnitude among all elements of that row. The reason for this is
that $\cp{U}{U}{AB}{}{}{}=\cp{U}{U}{B}{}{}{}\times_3 \mat{A}$ which implies
that the $m$-th view will be reconstructed mostly based on the $n$-th
frontal slice of $\cp{U}{U}{B}{}{}{}$. Then, we are matching the calculated
view clusters to the ground truth view clusters so that we can in turn
compare the calculated nodes clusters with the appropriate ground truth
nodes clusters. To this end, for each calculated view cluster we create a
membership vector containing either a $1$ or a $0$ at the $m$-th position if
the $m$-th view belongs or does not belong, respectively, to that cluster,
and then we normalize it to have unit norm. Similarly, we also create
membership vectors for the ground truth view clusters and we match each
calculated view cluster to the ground truth view cluster with which the
inner product of their corresponding membership vectors is maximized.

\paragraph{Clustering Performance Evaluation.}

To assess the performance of each method we will use the Adjusted Mutual
Information (AMI). AMI is an adjusted version of the Normalized Mutual
Information (NMI) designed to mitigate the flaw of NMI of getting larger
values as the number of clusters gets closer to the total number of
samples\reminder{cite}. Note that, although not reported here, in our
experiments NMI and the Adjusted Rand Index (ARI) produced very similar
results to AMI. The only notable difference is that sometimes the original
Symmetric Richcom gives better ARI score than the original CMNC, which is
not the case with AMI. However, since, as we also confirm from the
experiments in \autoref{sec:artif_cluster_perf:results}, Symmetric Richcom
and CMNC have the worst performance among all methods, the presentation of
NMI and ARI scores is omitted.

\subsubsection{Real-World Case Study.}
\label{sec:real_world_1}

We preprocess the data by removing all airlines that offered less than 100
flights and then all airports corresponding to less than 30 flights, and we
repeat this process until no further airline or airport is removed. This
leaves us with 235 airports and 61 airlines whose flight counts are then
organized into a tensor of size $235\times 235 \times 61$. Note that since
each of the 61 frontal slices of this tensor corresponds to the flight
network of a specific airline, while the view clusters in this case will
represent clusters of airlines. Similarly, the 235 rows and columns of each
frontal slice correspond to the different airports between which the
corresponding airline has been flying. Therefore, each node clustering will
correspond to a clustering of airports.

\begin{figure*}[!htbp]
    \centering
    \subcaptionbox{Hour clusters\label{fig:real_mine_views}}[0.15\textwidth]{
        \includegraphics[clip,trim=6cm 0cm 6cm 0cm,height=0.425\textwidth]{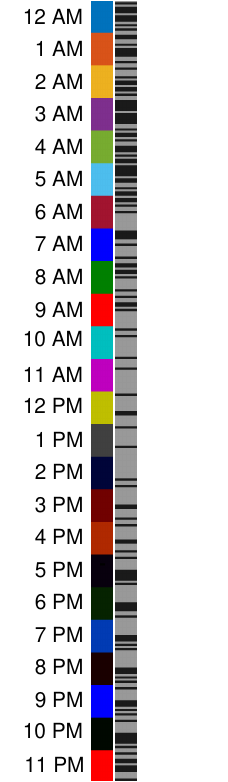}
    }%
	\hfill
    \subcaptionbox{Clusters of people during the day\label{fig:real_mine_nodes_2}}[0.425\textwidth]{
        \includegraphics[clip,trim=1.5cm 0cm 1.5cm 0cm,height=0.425\textwidth]{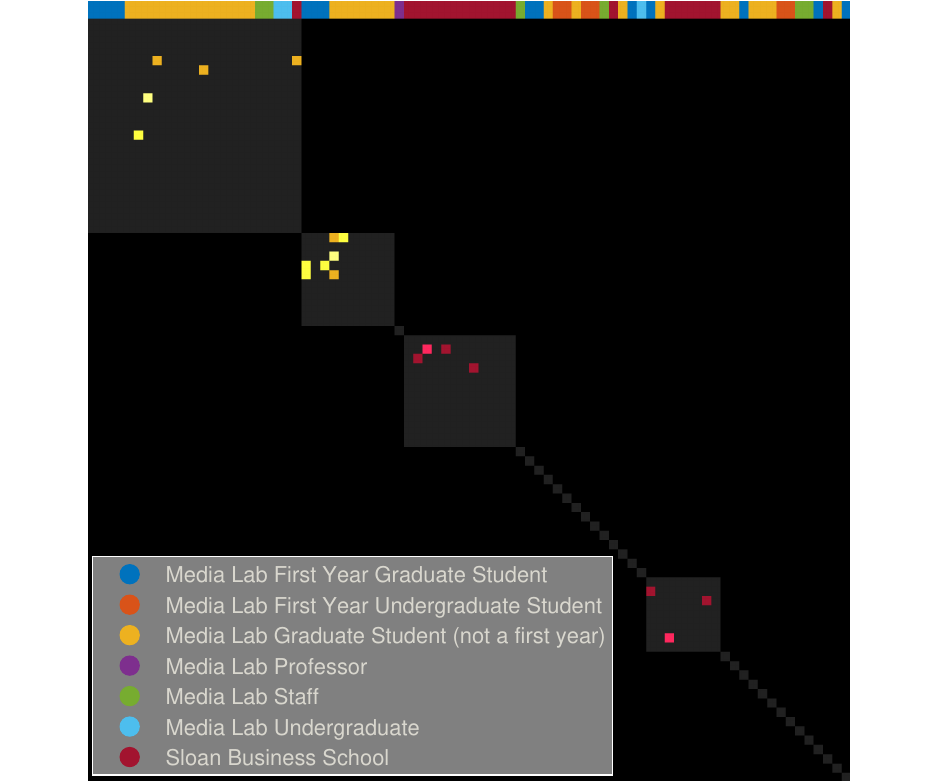}
    }%
	\hfill
    \subcaptionbox{Clusters of people during the night\label{fig:real_mine_nodes_1}}[0.425\textwidth]{
        \includegraphics[clip,trim=1.5cm 0cm 1.5cm 0cm,height=0.425\textwidth]{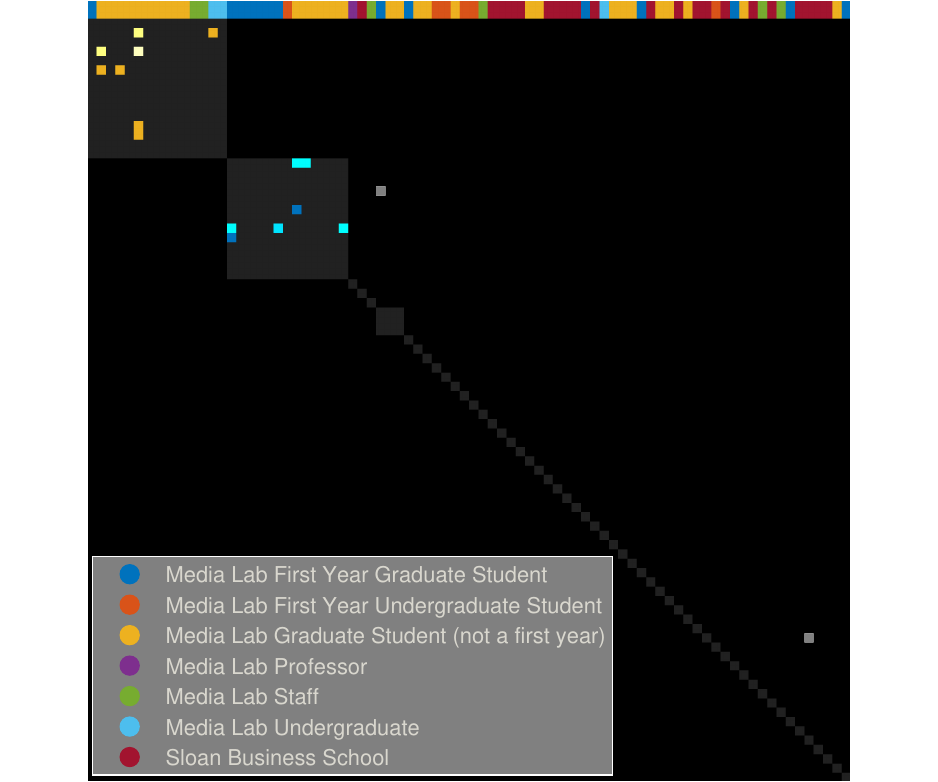}
        }%
	\null
    \caption{Clustering of hours of day and people by \method{} for the
    reality mining dataset. The right-hand colored bar in (a) shows the hour
    clusters while (b)-(c) show the adjacency matrices of a selected hour
    from the day and night clusters, respectively. The remaining color bars
    in (a)-(c) serve as ground truth labels.}
    \label{fig:real_mine}
\end{figure*}

\subsubsection{Execution Time on Artificial Data.}
\label{sec:exec_time_setup}
The quantities we will study are the number of nodes, $I$, the number of
views, $K$, the node embeddings dimension, $R$, and view embeddings
dimension, $M$. When varying $I$, we set $K$, $R$ and $M$ to 9, 3 and 3,
respectively. When varying $K$, we set $I$, $R$ and $M$ to 60, 3 and 3,
respectively. When varying $R$, we set $I$, $K$, and $M$ to 240, 9 and 3,
respectively. When varying $M$ we set $I$, $K$ and $R$ to 120, 9 and 96,
respectively. We also understand, that some could argue that these numbers
should have been larger for a proper time complexity analysis. However,
doing so would force us to completely exclude CMNC from the comparisons,
since it would very quickly produce out-of-memory errors. In fact, in our
experiments, this issue sometimes occured even with the aforementioned
experiment setup.

The graph generation process is similar to the one in
\autoref{sec:artif_clust_perf}, but here the intra-community edge
density, $\gamma$, is always fixed to 0.15. When altering the number of
nodes or views, the sizes of node clusters and view clusters are scaled
proportionally. Similarly, the parameters of the embedding generation is
parametrized in the same way as in
\autoref{sec:artif_clust_perf_appendix}, but with a few modifications.
First, we consider three equally spaced values instead of six for ComClus
for both its $\beta$ and $\rho$. Second, we consider a fixed threshold of
$1e-6$ for all methods. And third, we calculate 5 samples for each parameter
combination instead of 100.

\subsection{Real-World Case Study 2.}
Here we study the well known reality mining dataset \cite{eagle2006reality}
which documented the interactions of a group of students and faculty from
the MIT Media Laboratory and MIT Sloan business school via special software
on their phones. In our experiments, we aggregate all communications between
82 participants on an hourly basis and for a total of 15 days, which leads
to a time-evolving graph with an adjacency tensor of size $82\times 82\times
360$.

\autoref{fig:real_mine_views} indicates that the communication patterns of
the participants during the day have a slight but observable tendency to be
different from the communication patterns at night.  Specifically,
\autoref{fig:real_mine_nodes_2} shows that during the day there exist 2
major clusters consisting mostly of Media Lab Graduate students and 2
clusters of Sloan Business School students. On the other hand, from
\autoref{fig:real_mine_nodes_1} we observe that during the night the
clusters of Media Lab students shrink in size, while the clusters of Sloan
Business School students vanish altogether. This is in accordance with our
intuition that during the night the intensity communications is expected to
be lower.


\appendixcontent
\printbibliography
\end{refsection}
\end{document}
